\documentclass{article}

\usepackage{fullpage}
\usepackage{amssymb}
\usepackage{amsfonts}
\usepackage{mathrsfs}
\usepackage{amsthm}
\usepackage{amsmath}
\usepackage{mathtools}
\usepackage{bbm}
\usepackage{caption}
\usepackage{subcaption}
\usepackage{multirow}
\usepackage{graphicx}
\usepackage{booktabs}
\usepackage[export]{adjustbox}
\usepackage{pgfplots}
\pgfplotsset{compat=1.15}

\DeclareMathOperator*{\jac}{\operatorname{Jac}}

\newcommand{\R}{{\cal R}}
\newcommand{\Q}{{\cal Q}}
\newcommand{\reals}{\mathbb{R}}

\newtheorem{assumption}{Assumption}
\newtheorem{proposition}{Proposition}

\begin{document}

\title{Direct Estimation of Appearance Models for
  Segmentation\thanks{This material is based upon work supported by
    the National Science Foundation under Grant No. DMS-1439786 while
    the authors were in residence at the Institute for Computational
    and Experimental Research in Mathematics in Providence, RI, during
    the Spring 2019 semester.}}

\author{Jeova F. S. Rocha Neto\thanks{School of Engineering, Brown
    University, Providence, RI, USA} \and Pedro
  Felzenszwalb\footnotemark[2] \and Marilyn
  Vazquez\thanks{Mathematical Bioscience Institute, Ohio State
    University, Columbus, OH, USA}}

\maketitle

\begin{abstract}
Image segmentation algorithms often depend on appearance models that
characterize the distribution of pixel values in different image
regions.  We describe a new approach for estimating appearance models
directly from an image, without explicit consideration of the pixels
that make up each region.  Our approach is based on novel algebraic
expressions that relate local image statistics to the appearance of
spatially coherent regions.  We describe two algorithms that can use
the aforementioned algebraic expressions to estimate appearance models
directly from an image.  The first algorithm solves a system of linear
and quadratic equations using a least squares formulation.  The second
algorithm is a spectral method based on an eigenvector computation.
We present experimental results that demonstrate the proposed methods
work well in practice and lead to effective image segmentation
algorithms.
\end{abstract}



\section{Introduction}

Image segmentation is a fundamental problem in computer vision and
image processing.  The goal of segmentation is to partition an image
into regions corresponding to different objects or materials that are
visible in a scene.  There are numerous applications of automatic
segmentation methods including in medical image analysis, remote
sensing, industrial inspection, object recognition and interactive
image editing.  Therefore, advances in image segmentation can have
significant impact in many applications.

Segmentation algorithms often rely on some type of appearance modeling
to be able to classify each pixel in the image into different regions.
Many segmentation algorithms are based on minimization of functionals
that encourage spatial coherence of regions and a fit of the pixels in
each region to an appearance model.

In some settings such as medical image analysis or remote sensing the
appearance of different objects of interest may be known in advance.
On the other hand, in many applications the appearance of the objects
or materials in a scene need to determined during the segmentation
process.

When appearance models are not known in advance, the challenge of
image segmentation is that it is essentially a ``chicken-and-egg''
problem.  If we had appearance models we could segment the image.  On
the other hand, if we had a segmentation we could estimate the
appearance of each region.  Some methods such as
\cite{rother2004grabcut} and \cite{tang2014pseudo} alternate between
estimating appearance models and estimating a segmentation in an
iterative fashion.  Similarly, variational methods such as
\cite{chan2001active} and \cite{ni2009local} repeatedly evolve a
segmentation and the appearance model of each region.  Less
frequently, methods such as \cite{vicente2009joint} and
\cite{tang2013grabcut} tackle the segmentation problem by making the
dependency on appearance models implicit.  Other methods such as the
graph-based approaches in \cite{shi2000normalized} and
\cite{felzenszwalb2004efficient} avoid using appearance models and
instead use pairwise similarity to cluster pixels.

In this paper we describe a novel approach for estimating appearance
models directly from an image, without explicit consideration of the
pixels that make up each region.  The approach relies on the spatial
coherence of regions but does not require solving for a segmentation
to determine the appearance of each region.  Instead we derive a set
of algebraic expressions that can be used to solve for the appearance
models of each region using local statistics that can be easily
estimated directly from an image.

We focus specifically on the problem of \emph{binary} segmentation
where the goal is to partition an image into two regions (often the
foreground and background).  Our results allow for the development of
new segmentation algorithms that use a two-step process.  First we
determine appearance models without segmenting the image.  We can then
segment the image using any available method that requires appearance
models to be fixed in advance.  In practice we propose to use the
estimated appearance models to define a Markov Random Field
\cite{besag1986statistical}, and use efficient graph cut algorithms to
segment the image \cite{greig1989exact,boykov1999fast}.  This two-step
approach leads to efficient methods for unsupervised segmentation.

Many image segmentation algorithms pre-process an image using the
response of Gabor filters to define features
\cite{jain1991unsupervised, mccann2014images, yuan2015factorization}.
Although very useful for capturing texture information, the use of
filter banks requires significant processing and memory
overhead. Moreover, the use of filter responses can lead to inaccurate
localization of region boundaries, due to the non-trivial spatial
extent of texture filters.

Recent segmentation methods have also use machine learning techniques
to define local features or in complete end-to-end systems (see,
e.g.\ \cite{minaee2021image}).  Machine learning methods typically
require labeled datasets for training and often do not generalize
across different imaging domains.  The resulting systems can also
suffer from limited interpretability.  In contrast, our model based
approach is easily interpretable and can be applied to different types
of images without requiring any training data.

The algorithms we describe in this paper do not rely on filter
responses or learned features.  Instead we work directly with raw
pixel values.  Our methods estimate non-parametric models that
represent the appearance of regions in terms of distributions over
pixel values.  This makes it possible to accurately localize the
boundaries between regions with complex appearances.

The remainder of the paper is organized as follows. In Section 2 we
describe how we represent the appearance of regions using
non-parametric models that capture the distribution of
pixel values within a region.  In Section 3 we discuss how local image
statistics can be related to the appearance models of images with two
regions using a set of algebraic constraints.  These algebraic
constraints lead to two different methods for estimating appearance
models described in Section 4.  In Section 5, we show experimental
results and evaluate both the accuracy of the estimated appearance
models and the segmentations obtained using these models.  Finally,
Section 6 concludes the paper with a summary and brief discussion.

\section{Appearance Models for Segmentation}

Let $I : \Omega \rightarrow L$ be an image, where $\Omega$ is a set of
pixel locations and $L$ is a finite set of pixel values.  For an $n$
by $m$ graylevel image we have $\Omega = \{1,\ldots,n\} \times
\{1,\ldots,m\}$ while $L=\{0,\ldots,255\}$.  In Section
\ref{sec:realimgs}, we will discuss how to generalize our methods to
RGB and vector valued images.  We use $I(x)$ to denote the value of a
pixel $x \in \Omega$.  A region $\R$ is a subset of the pixels in
$\Omega$.

We define an \emph{appearance model} $\theta \in \reals^L$ to be a
distribution over $L$.  An appearance model for a region $\R$
specifies the typical values for the pixels in $\R$.  Note, however,
that an appearance model $\theta$ does not fully specify the joint
distribution of the pixels in a region, only the probabilities of
observing different pixel values.  We do not assume the pixel values
in each region are independent.  Therefore the appearance models
considered here only define a coarse representation for the appearance
of a region.  For example, if we permute the pixels within a region
the appearance model would remain the same.

We assume the image domain $\Omega$ can be divided into two regions 
$\R_0$ and $\R_1$ and use $S : \Omega \rightarrow \{0,1\}$ to denote 
a binary image specifying an assignment of image pixels into regions, 
or a \emph{segmentation}.  Let $\theta_0$ and $\theta_1$
be appearance models for $\R_0$ and $\R_1$ respectively.

As discussed in the introduction, one of the fundamental challenges of
unsupervised image segmentation is that it is essentially a
``chicken-and-egg'' problem.  If we have appearance models $\theta_0$
and $\theta_1$ we can partition the image into regions using several
existing approaches.  On the other hand, if we have a partition of an
image into regions $\R_0$ and $\R_1$ we can estimate appearance models
by computing normalized histograms of the pixel values in each region.

The main contribution of our work is a method for estimating the
appearance models $\theta_0$ and $\theta_0$ \emph{directly} from an
image, without explicit consideration of the regions $\R_0$ and
$\R_1$.  

Once appearance models are estimated there are many methods that can
be used to partition an image.  In the simplest setting we can
classify each pixel independently using a likelihood ratio defined by
$\theta_0$ and $\theta_1$.  Alternatively, methods based on Markov
Random Fields (\cite{besag1986statistical}) segment an image by
minimizing an energy that combines the appearance models with a
boundary regularization term,
\begin{equation}\label{eq:E_seg}
    E(S | \lambda, \theta_0, \theta_1) =
    -\sum_{x \in \Omega} \ln \theta_{S(x)}(I(x)) + \lambda \sum_{x,y \text{ neighbors}} |S(x)-S(y)|.
\end{equation}
In this case the minimum energy segmentation corresponds to a MAP
estimate of $S$, where the boundary regularization defines a prior
over segmentations.  Importantly, the energy function $E$ can be
efficiently minimized by the computation of a minimum cut in a graph
(\cite{greig1989exact}).

\section{Statistical Model}
\label{sec:model}

In order to address the appearance model estimation problem, we treat
the image $I : \Omega \rightarrow L$ as a realization of a random field.
Let $\R_0$ and $\R_1$ be a partition of $\Omega$ into two regions.
Let $\theta_0$ and $\theta_1$ be appearance models associated with
$\R_0$ and $\R_1$ respectively.

Our approach relies on two key assumptions.  The first assumption is
that regions have homogeneous appearance in the following sense.

\begin{assumption}[Homogeneity]
  \label{as:homogeneity}
  The probability that a pixel $x$ takes a particular value depends
  only on the region the pixel belongs to,
  $$P(I(x)=i \, | \, x \in \R_s) = \theta_s(i).$$ 
\end{assumption}

Note that this assumption does not specify a full generative model for
the image.  We assume the pixels in each region have the same marginal
distribution, but their joint distribution could involve dependencies,
as we see for example in images with textures.  Similar assumptions
have been used in other approaches for image segmentation, including
several methods for unsupervised texture segmentation
\cite{ni2009local, yuan2015factorization, mccann2014images}.

In general we expect the values of nearby pixels to be
dependent.  For example, neighboring pixels that are in the same
region almost always have similar values.  
The second assumption we make captures the idea that 
sufficiently far away pixels are independent.  

\begin{assumption}[Independence at a distance]
  \label{as:independence}
  If $x$ and $y$ are two pixels at a distance $r$ from each other,
  $$P(I(x)=i,I(y)=j \, | \, x \in \R_s, y \in \R_t) =
  P(I(x) = i \, | \, x \in \R_s) P(I(y) = j \, | \, y \in \R_t).$$
\end{assumption}

To estimate $\theta_0$ and $\theta_1$ we consider two
distributions $\alpha$ and $\beta$
that can be directly estimated from an observed image.
We will relate $\theta_0$ and $\theta_1$ to $\alpha$ and $\beta$ using
the assumptions above.

Let $x, y \in \Omega$ be a pair of pixels at a distance $r$ 
from each other, selected uniformly at random.  
Since the pixels are in a discrete
grid we use the L1 norm to measure the distance between them.
Let $\alpha \in \reals^L$ be a distribution over $L$ where $\alpha(i)$
is the probability that pixel $x$ has value $i$.  Let $\beta \in
\reals^{L \times L}$ be a distribution over $L \times L$ where
$\beta(i,j)$ is the probability that pixel $x$ has value $i$ and pixel $y$
has value $j$.

$$Q = \{ (x,y) \in \Omega^2 \,|\, ||x-y|| = r\}.$$
$$(x,y) \sim \text{Uniform}(Q).$$
$$\alpha(i) = P(I(x) = i).$$
$$\beta(i,j) = P(I(x) = i, I(y) = j).$$

Note that we can easily estimate $\alpha$ and $\beta$ from an observed
image.  We simply enumerate all pairs of pixels at distance $r$ from
each other and count the number of times we observe pixels with
particular values.  If the image is large we can also estimate the two
distributions using a random sample of pairs of pixels at distance
$r$.  We use $\hat{\alpha}$ and $\hat{\beta}$ to denote the estimates
of $\alpha$ and $\beta$ computed from an observed image.
$$\hat{\alpha}(i) = \frac{1}{|Q|}\sum_{(x,y) \in Q} \mathbbm{1}(I(x) = i).$$
$$\hat{\beta}(i,j) = \frac{1}{|Q|}\sum_{(x,y) \in Q} \mathbbm{1}(I(x) = i) \mathbbm{1}(I(y) = j).$$

\begin{figure}[t]
    \centering
    \begin{tikzpicture}
        \begin{axis}[
            scale only axis,
            height=3cm,
            width=.9\textwidth,
            xlabel={$r$},
            ylabel={$d_{B}( \hat{\beta},\hat{\alpha}\hat{\alpha}^{\top})$},
            yticklabel style = {font=\footnotesize},
            xticklabel style = {font=\footnotesize},
            ylabel near ticks,
            xlabel near ticks,
            xmin=0,xmax=50,
            ]
        \addplot[mark=none, black, thick] table[x index=0,y index=1] {data/beta_graph_one_tex_new.dat};
        \end{axis}
    \end{tikzpicture}
    \caption{Evaluating independence at a distance for images with a
      single Brodatz texture (see Figure~\ref{fig:brodatz_patt} for
      examples of the images in this dataset).  We show the average
      Bhattacharyya distance between $\hat{\beta}$ and
      $\hat{\alpha}\hat{\alpha}^{\top}$ as a function of $r$.}
    \label{fig:beta_graph_one_tex}
\end{figure}
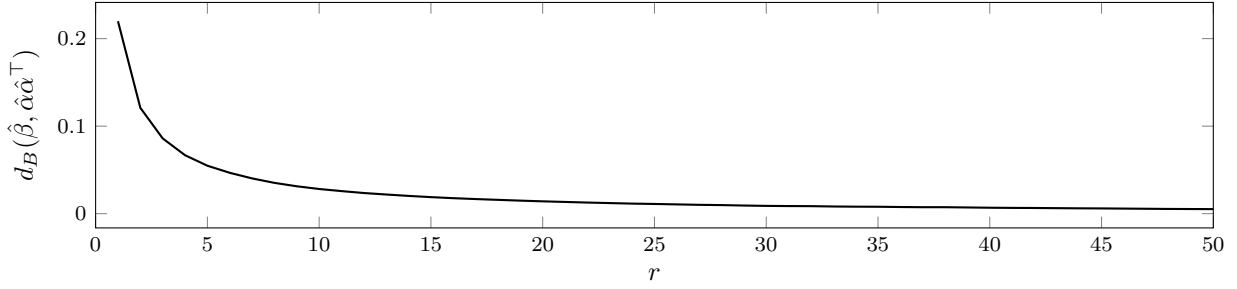

For images with a single region Assumption~\ref{as:independence} leads
to $\beta = \alpha \alpha^{\top}$.
Figure~\ref{fig:beta_graph_one_tex} evaluates this identity in
textured images from the Brodatz dataset \cite{brodatz1966textures}.
See Figure \ref{fig:brodatz_patt} for examples of the images in the
Brodatz dataset.  In this case each image has a single textured region
and we compare $\hat{\beta}$ to $\hat{\alpha} \hat{\alpha}^{\top}$ for
different values of $r$ using the Bhattacharyya distance (see
Section~\ref{sec:eval}).  When $r$ is small we see that the two
distributions, $\hat{\beta}$ and $\hat{\alpha} \hat{\alpha}^{\top}$,
are quite different, because nearby pixels are not independent.  As we
increase $r$ we see that $\hat{\beta}$ is close to $\hat{\alpha}
\hat{\alpha}^{\top}$, suggesting that pixels that are relatively far
from each other are independent.

Let $w_0 = P(x \in \R_0)$ and $w_1 = P(x \in \R_1)$.  Now consider the
probability that $x$ and $y$ are in different regions.
Let $\epsilon = P(x \in \R_0, y \in \R_1) = P(x
\in \R_1, y \in R_0)$.  The probability that $x$ and $y$ are in
different regions is $2\epsilon.$

The following proposition derives a set of linear and quadratic
algebraic constraints
between $\alpha$, $\beta$, $w_0$, $w_1$, $\theta_0$, $\theta_1$ and
$\epsilon$.  These constraints will enable us to estimate $\theta_0$
and $\theta_1$ without explicit consideration of $\R_0$ and $\R_1$.

\begin{proposition}
\label{prop:constraints}
Under Assumption~\ref{as:homogeneity} and Assumption~\ref{as:independence} we have:
\begin{equation}
    \alpha = w_0 \theta_0 + w_1 \theta_1.
\label{eq:alpha}
\end{equation}
\begin{equation}
\beta = (w_0-\epsilon) \theta_0 \theta_0^{\top} + (w_1-\epsilon) \theta_1 \theta_1^{\top} + \epsilon \theta_0 \theta_1^{\top} + \epsilon \theta_1 \theta_0^{\top}.
\label{eq:beta}
\end{equation}
\end{proposition}
\begin{proof}
The result follows from the law of total probability.
For (\ref{eq:alpha}),
\begin{align*}
    \alpha(i) & = P(I(x) = i) \\
    & = P(x \in \R_0) P(I(x) = i \,|\, x \in \R_0) + P(x \in \R_1) P(I(x) = i \,|\, x \in \R_1)\\
    & = w_0 \theta_0(i) + w_1 \theta_1(i).
\end{align*}

For (\ref{eq:beta}), first note that $$P(x \in \R_0) = P(x \in
\R_0, y \in \R_0) + P(x \in \R_0, y \in \R_1).$$ Therefore $P(x \in
\R_0, y \in \R_0) = w_0 - \epsilon$.  Similarly $P(x \in \R_1, y \in
\R_1) = w_1 - \epsilon$.  Now,
\begin{align*}
    \beta(i,j) & =  P(I(x) = i,I(y)=j) \\
    & = \begin{aligned}[t]
    & P(x \in \R_0, y \in \R_0) P(I(x) = i, I(y) = j \,|\, x \in \R_0, y \in \R_0) \,+ \\
     & P(x \in \R_1, y \in \R_1) P(I(x) = i, I(y) = j \,|\, x \in \R_1, y \in \R_1) \,+ \\
     & P(x \in \R_0, y \in \R_1) P(I(x) = i, I(y) = j \,|\, x \in \R_0, y \in \R_1) \,+ \\
     & P(x \in \R_1, y \in \R_0) P(I(x) = i, I(y) = j \,|\, x \in \R_1, y \in \R_0) 
     \end{aligned} \\
    & =  (w_0-\epsilon) \theta_0(i) \theta_0(j) 
    + (w_1-\epsilon) \theta_1(i) \theta_1(j) 
    + \epsilon \theta_0(i) \theta_1(j) 
    + \epsilon \theta_1(i) \theta_0(j).
\end{align*}

\end{proof}

\subsection{Extension to images with $m$ regions}
\label{sec:multi}

It is possible to generalize the algebraic constraints in
Proposition~\ref{prop:constraints} to the case of images with $m$
regions.  Let $\R_0,\ldots,\R_{m-1}$ be a partition of $\Omega$ into
$m$ regions.  In this case we have parameters $\theta_i$, $w_i$ and
$\epsilon_{ij}$ for $0 \le i,j \le m-1$ and $j \neq i$.  Here
$\theta_i$ is the appearance model associated with $\R_i$ while $w_i = P(x
\in \R_i)$ and $\epsilon_{ij} = P(x \in \R_i, y \in \R_j)$.  Since $x$
and $y$ are exchangeable $\epsilon_{ij} = \epsilon_{ji}$.

Using the law of total probability we can derive
the following algebraic expressions,
\begin{equation}
    \alpha = \sum_{i=0}^{m-1} w_i \theta_i,
\end{equation}
\begin{equation}
    \beta = \sum_{i=0}^{m-1} \left( (w_i - \sum_{j \neq i} \epsilon_{ij}) \theta_i \theta_i^{\top} + \sum_{j \neq i} \epsilon_{ij} \theta_i \theta_j^{\top}\right).
\end{equation}\vspace{.2cm}

We note, however, that the algorithms we develop to estimate
appearance models when $m=2$ do not readily generalize to the case
when $m >2$.  For the remainder of the paper we consider only the case
of images with 2 regions.

\section{Appearance Estimation}

As discussed in the previous section we can estimate $\alpha$ and
$\beta$ directly from an image.  In this section we show how we can
recover $\theta_0$ and $\theta_1$ from $\alpha$ and $\beta$ or their
estimates $\hat{\alpha}$ and $\hat{\beta}$.  Initially we assume that
$w_0$, $w_1$ and $\epsilon$ are known, and present two different
methods (Sections~\ref{sec:alg} and~\ref{sec:spec}) for estimating
$\theta_0$ and $\theta_1$.  We consider the case where $w_0$, $w_1$
and $\epsilon$ are unknown in Section~\ref{sec:w0w1e}.

Our methods require that $w_0 w_1 \neq \epsilon$.
Proposition~\ref{prop:spectral} in Section~\ref{sec:spec} shows that
if $w_0 w_1 = \epsilon$ then $\beta = \alpha \alpha^{\top}$.  In this
case we do not have enough information to recover the appearance
models.  In general we expect that $w_0 w_1 > \epsilon$ as long $r$ is
not too large.  This condition imposes certain constraints of the
shape and size of the regions $\R_0$ and $\R_1$ relative to the value
of $r$.

\subsection{Algebraic Method}
\label{sec:alg}

Suppose $\alpha$, $\beta$, $w_0$, $w_1$ and $\epsilon$ are known, and
consider the problem of recovering $\theta_0$ and $\theta_1$.

Let $L = \{1,\ldots,k\}$.  We have $2k$ unknowns
$\theta_0(1),\ldots,\theta_0(k)$ and $\theta_1(1),\ldots,\theta_1(k)$.
Proposition~\ref{prop:constraints} defines $k$ linear and $k^2$
quadratic constraints,
\begin{eqnarray}
& \alpha(i) = w_0 \theta_0(i) + w_1 \theta_1(i), \\
& \beta(i,j) = (w_0-\epsilon) \theta_0(i) \theta_0(j) + (w_1-\epsilon) \theta_1(i) \theta_1(j) + \epsilon \theta_0(i) \theta_1(j) + \epsilon \theta_1(i) \theta_0(j).
\end{eqnarray}

Since $\theta_0$ and $\theta_1$ define probability distributions we
also have the additional constraints that both vectors should sum to
one and $\theta_0(i) \ge 0$, $\theta_1(i) \ge 0$ for $1 \le i \le k$.

\subsubsection{Minimal constraints}

We first consider a simple method that uses a subset of the
constraints to solve for all of the unknowns in the appearance
models. The approach uses the $k$ linear constraints defined by
$\alpha$ and $k$ quadratic constraints defined by a single row of
$\beta$.  For the derivation below we treat $\alpha$ and
$\beta$ as known although in practice we only have empirical 
estimates of the two distributions.  

\begin{enumerate}
\item Let $i \in L$.  Using the quadratic constraint defined by
  $\beta(i,i)$ and the linear constraint defined by $\alpha(i)$ we can
  solve for $\theta_0(i)$ and $\theta_1(i)$,
\begin{equation}
{\theta}_0(i) = {\alpha}(i) \pm \frac{\sqrt{\frac{w_0 w_1 - \epsilon}{\omega_1^2} (\beta(i,i)-\alpha^2(i))}}{\frac{w_0 w_1 -\epsilon}{w_1^2}}, \;
{\theta}_1(i) = \frac{{\alpha}(i) - w_0{\theta_0}(i)}{w_1}.
\end{equation} 

\item Now consider each $j \in L$ with $j \neq i$.  Since we solved
  for $\theta_0(i)$ and $\theta_1(i)$ in Step 1, now $\beta(i,j)$
  defines a linear constraint on $\theta_0(j)$ and $\theta_1(j)$.
  Together with the linear constraint defined by $\alpha(j)$ we can
  solve for $\theta_0(j)$ and $\theta_1(j)$,
\begin{equation}\label{eq:entryj}
{\theta}_0(j) = \frac{w_1 \beta(i,j)-\alpha(j)(w_1{\theta}_1(i)+\epsilon({\theta}_0(i)-{\theta}_1(i)))}{(w_0w_1-\epsilon)({\theta}_0(i)-{\theta}_1(i))}, \;\;
{\theta}_1(j) = \frac{{\alpha}(j) - w_0{\theta_0}(j)}{w_1}.
\end{equation}
\end{enumerate}

When $\beta(i,i) = \alpha^2(i)$ we have $\theta_0(i) = \theta_1(i)$.
To avoid dividing by zero when solving for $\theta_0(j)$ in Step 2 and
to increase the robustness of the method, we can select $i$ maximizing
$\beta(i,i)-\alpha^2(i)$ in Step 1.  Note that if
$\beta(i,i)-\alpha^2(i) = 0$ for all $i$ then $\theta_0 = \theta_1 =
\alpha$.

To solve for $\theta_0(i)$ in Step 1 we require that 
$(w_0 w_1 - \epsilon)(\beta(i,i)-\alpha^2(i)) \ge 0$.  Proposition
\ref{prop:spectral} below shows that under the assumptions we have made
\begin{equation} \label{eq:beta_minus_alpha}
    \beta(i,i)-\alpha^2(i) = (w_0w_1 - \epsilon)(\theta_0(i) - \theta_1(i))^2.
\end{equation}
Therefore $(w_0 w_1 - \epsilon)(\beta(i,i) - \alpha^2(i)) \geq 0$.

\subsubsection{Least squares solution}

The approach described above uses a small number of the constraints
defined by $\alpha$ and $\beta$ to exactly recover $\theta_0$ and
$\theta_1$.  However, in practice we only have empirical estimates of
$\alpha$ and $\beta$.  We also don't expect real data to perfectly fit
our assumptions.  We now describe an alternative method that uses all
of the constraints defined by $\alpha$ and $\beta$ in a least squares
formulation.

Let $i_1,\ldots,i_k$ be an ordering of $L=\{1,\ldots,k\}$.  Our
empirical results show that ordering the indices in decreasing
value of $\hat{\beta}(i,i)-\hat{\alpha}^2(i)$ works well and is 
better than a random order.

\begin{enumerate}
\item We start by solving for $\theta_0(i_1)$ and $\theta_1(i_1)$
  using the quadratic constraint
  defined by $\hat{\beta}(i_1,i_1)$ and the linear constraint
  defined by $\hat{\alpha}(i_1)$.
\begin{equation}
{\theta}_0(i_1) = \hat{\alpha}(i) \pm \frac{\sqrt{\frac{w_0 w_1 - \epsilon}{\omega_1^2} (\hat{\beta}(i_1,i_1)-\hat{\alpha}^2(i_1))}}{\frac{w_0 w_1 -\epsilon}{w_1^2}}, \;
{\theta}_1(i_1) = \frac{\hat{\alpha}(i_1) - w_0{\theta_0}(i_1)}{w_1}.
\end{equation} 

\item We iterate $\ell$ from $2$ to $k$ and solve for
  $\theta_0(i_\ell)$ and $\theta_1(i_\ell)$ in each step.  When
  solving for $\theta_0(i_\ell)$ and $\theta_1(i_\ell)$ we already
  have values for $\theta_0(i_1),\ldots,\theta_0(i_{\ell-1})$ and
  $\theta_1(i_1),\ldots,\theta_1(i_{\ell-1})$.  Therefore
  $\hat{\beta}(i_1,i_\ell),\ldots,\hat{\beta}(i_{\ell-1},i_\ell)$
  define $\ell-1$ linear constraints on $\theta_0(i_\ell)$ and
  $\theta_1(i_\ell)$.  Together with the constraint defined by
  $\hat{\alpha}(i_\ell)$ we form a system with $\ell$ linear equations
  and 2 unknowns that can be solved using linear least squares:
  $$\forall j \in \{i_1,\ldots,i_{\ell-1}\}\;\; ((w_0-\epsilon) \theta_0(j) + \epsilon \theta_1(j)) \theta_0(i_\ell) + ((w_1 - \epsilon) \theta_1(j) + \epsilon \theta_0(j)) \theta_1(i_\ell) = \hat{\beta}(j, i_\ell),$$
  $$w_0 \theta_0(i_\ell) + w_1 \theta_1(i_\ell) = \hat{\alpha}(i_\ell).$$

\item We improve our estimates by iterating $\ell$ from $1$ to $k$ and
  re-estimate $\theta_0(i_\ell)$ and $\theta_1(i_\ell)$ in each step.
  To re-estimate $\theta_0(i_\ell)$ and $\theta_1(i_\ell)$ we use
  $\hat{\beta}(j,i_\ell)$ and the current values for $\theta_0(j)$ and
  $\theta_1(j)$ for $j \neq i_\ell$ to define $k-1$ linear
  constraints.  Together with the constraint defined by
  $\hat{\alpha}(i_\ell)$ we form a system with $k$ linear equations
  and 2 unknowns that can be solved using linear least squares:
  $$\forall j \neq i_\ell\;\; ((w_0-\epsilon) \theta_0(j) + \epsilon \theta_1(j)) \theta_0(i_\ell) + ((w_1 - \epsilon) \theta_1(j) + \epsilon \theta_0(j)) \theta_1(i_\ell) = \hat{\beta}(j, i_\ell)$$
  $$w_0 \theta_0(i_\ell) + w_1 \theta_1(i_\ell) = \hat{\alpha}(i_\ell)$$
  
  Empirically we found that iterating over the entries one time using
  this method is enough to obtain improved results.

\item We set $\theta_s(i) = \max(\theta_s(i),0)$ and normalize
  $\theta_0$ and $\theta_1$ to add up to one.  This ensures $\theta_0$
  and $\theta_1$ define valid probability distributions.
\end{enumerate}

Note that there are two choices for the value of $\theta_0(i_1)$ when
solving the quadratic equation in Step 1.  We consider both choices to
estimate full appearance models.  We then compare $\beta$ defined by
the estimated models and Equation (\ref{eq:beta}) to $\hat{\beta}$
using the Bhattacharyya distance.  We select the appearance models
leading to the smaller Bhattacharyya distance.

\subsection{Spectral Method}
\label{sec:spec}

Now we describe a spectral method for estimating the appearance
models.  As in the previous section we assume the values of $w_0$, $w_1$ and
$\epsilon$ are known.  The following proposition provides the basis
for the approach.

\begin{proposition}
\label{prop:spectral}
\begin{equation}
    \beta - \alpha\alpha^\top = \left(w_0 w_1 - \epsilon\right) \left(\theta_0-\theta_1\right)\left(\theta_0-\theta_1\right)^\top
\end{equation}
\end{proposition}
\begin{proof}
First note that $w_0 + w_1 = 1$ implies $w_0 w_1 = w_0 - w_0^2$ and $w_0 w_1 = w_1 - w_1^2$.
\begin{align*}
 \beta - \alpha \alpha^{\top} & = (w_0 - \epsilon)\theta_0\theta_0^{\top} + (w_1 - \epsilon)\theta_1\theta_1^{\top} + \epsilon \theta_0\theta_1^{\top} + \epsilon\theta_1\theta_0^{\top} - (w_0 \theta_0 + w_1 \theta_1)(w_0 \theta_0 + w_1 \theta_1)^{\top} \\
 & = (w_0 - \epsilon - w_0^2)\theta_0 \theta_0^{\top} + (w_1 - \epsilon - w_1^2) \theta_1 \theta_1^{\top} + (\epsilon - w_0 w_1) \theta_0 \theta_1^{\top} + (\epsilon - w_1 w_0) \theta_1 \theta_0^{\top} \\
 & = (w_0 w_1 - \epsilon)\theta_0 \theta_0^{\top} + (w_0 w_1 - \epsilon) \theta_1 \theta_1^{\top} - (w_0 w_1- \epsilon) \theta_0 \theta_1^{\top} - (w_1 w_0 - \epsilon) \theta_1 \theta_0^{\top} \\
  & = (w_0 w_1 - \epsilon)(\theta_0 - \theta_1) (\theta_0 - \theta_1)^{\top} 
 \end{align*}
\end{proof}

Let $u = \theta_0 - \theta_1$.  The above result shows that that
matrix $\beta - \alpha\alpha^\top$ is of rank one and its only
eigenvector with non-zero eigenvalue is proportional to $u$.
Moreover, the corresponding eigenvalue is $(w_0 w_1 - \epsilon) || u
||^2$.

The matrix $\hat{\beta}-\hat{\alpha}\hat{\alpha}^{\top}$ defines an
approximation to $\beta - \alpha\alpha^\top$ that can be computed from an image.
Let $v$ be the dominant
eigenvector of $\hat{\beta}-\hat{\alpha}\hat{\alpha}^{\top}$
normalized so that $||v||=1$.  The vector $v$ gives us an estimate of
$u / || u ||$ up to a sign ambiguity.  The corresponding eigenvalue
$\lambda$ can be used to approximate $(w_0w_1 - \epsilon) ||u ||^2$.
We can estimate $u$ as,
\begin{equation}
    \hat{u} = \pm \sqrt{\frac{\lambda}{w_0w_1 - \epsilon}}v
    \label{eq:u}
\end{equation}
We can then use $\hat{\alpha}$ and Equation (\ref{eq:alpha}) 
to estimate $\theta_0$ and $\theta_1$,
\begin{eqnarray*}
\theta_0 & = \hat{\alpha} + w_1 \hat{u},\\
\theta_1 & = \hat{\alpha} - w_0 \hat{u}.
\end{eqnarray*}

Finally, we set $\theta_s(i) = \max(\theta_s(i),0)$ and normalize
$\theta_0$ and $\theta_1$ to add up to one.  This ensures $\theta_0$
and $\theta_1$ define valid probability distributions.

To handle the sign ambiguity in Equation (\ref{eq:u}) we consider both
choices to estimate $\theta_0$ and $\theta_1$.  We then compare
$\beta$ defined by the estimated appearance models and Equation
(\ref{eq:beta}) to $\hat{\beta}$ using the Bhattacharyya distance.  We
select the choice of sign in Equation (\ref{eq:u}) leading to the
smaller Bhattacharyya distance.

Notice that we can use simple power iteration methods to compute $v$
and $\lambda$.  The algorithm complexity depends on $|L|$, which
is usually much smaller that the number of pixels in the image.  The
rate of convergence of power iteration depends on the spectral gap,
which is proportional to $w_0 w_1 - \epsilon$.

\subsection{Estimating $w_0$, $w_1$ and $\epsilon$}
\label{sec:w0w1e}

The methods described above assume $w_0$, $w_1$ and $\epsilon$ are
known.  We have experimented with two different approaches for
estimating the appearance models when $w_0$, $w_1$ and $\epsilon$ are
unknown.  The first approach simply selects a typical, or average,
value for each of the unknown parameters.  The second approach
involves an explicit search over a discretized set of choices for the
unknown parameters.

\subsubsection{Typical values}

\begin{figure}[t]
    \centering
    \includegraphics[width=2in]{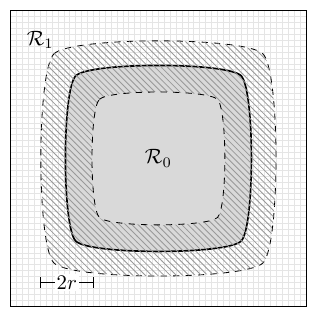}
    \caption{The area where pairs of pixels with $||x-y||=r$ can be in different regions.}
    \label{fig:epsilon}
\end{figure}

A simple approximation for $w_0$ and $w_1$ involves setting $w_0 = w_1
= 0.5$.  Although this is a crude approximation we have found that it
leads to good appearance models in a wide variety of images.

In order to approximate $\epsilon$, we start from the assumption
that the ground truth segmentation $S$ is spatially coherent and that
the boundary between regions $\partial S$ is short, i.e., $|\partial
S| \approx \sqrt{|\Omega|}$.  In this case $\epsilon$ is
proportional to the area within distance $r$ from $\partial S$ divided
by $|\Omega|$, see Figure~\ref{fig:epsilon}.  In our experiments we
set $r = \rho\sqrt{|\Omega|}$, where $\rho$ is a parameter set by the
user.  This makes the selection of the distance $r$ be adaptive to the
image resolution.  Our estimate of $\epsilon$ then becomes,
\begin{equation}\label{eq:epsilon_estimation}
\epsilon = \kappa\frac{r\sqrt{|\Omega|}}{|\Omega|} = \kappa \rho.
\end{equation}
We have found that setting $\kappa = 0.5$ often leads to good results
in practice.

\subsubsection{Searching over $w_0$, $w_1$ and $\epsilon$}

To search over $w_0$ and $w_1$ we use the fact that $w_0 + w_1 = 1$
and search over possible values for $w_0$.  We assume without
loss of generality that $w_0 \le w_1$ and $w_0 \in (0, 0.5)$.  In
practice we discretize the interval $(0,0.5)$ using a step of size of
$0.05$, leading to 11 choices for $w_0$.  To estimate $\epsilon$ we
search over the interval $(0,0.1)$ using a step size of $0.01$,
leading to 11 choices for $\epsilon$.  Together this leads to 121
combined choices for $w_0$ and $\epsilon$.

For each choice of parameters $w_0$, $w_1$ and $\epsilon$ we estimate
$\theta_0$ and $\theta_1$ using either the algebraic or spectral
method above.  We then compare $\beta$ defined by the estimated
appearance models and Equation (\ref{eq:beta}) to the empirical
$\hat{\beta}$ computed from the image.  We select the model parameters
minimizing the Bhattacharyya distance between $\beta$ and
$\hat{\beta}$.

Searching over $w_0$, $w_1$ and $\epsilon$ with the spectral method is
fairly efficient because the bottleneck in the spectral method is
computing the dominant eigenvector of
$\hat{\beta}-\hat{\alpha}\hat{\alpha}^{\top}$.  Since this matrix does
not depend on the unknown parameters the eigenvector only has to be
computed once.  In this case searching for the parameters leads to
limited overhead.  Searching for the parameters with the algebraic
method is much less efficient.  The experiments in the
Section~\ref{sec:experiments} evaluate the running time of the
different approaches.

\subsection{Examples}

Figures~\ref{fig:example_estimation_algebraic}
and~\ref{fig:example_estimation_spectral} illustrate the estimation of
appearance models on real images with the algebraic and spectral
methods respectively.  In these examples we used $\rho = 0.06$ to
select $r$.  The values of $w_0$, $w_1$ and $\epsilon$ were estimated
separately for each image by searching over discrete choices as
described in the last section.  For comparison we also show the
appearance models computed using ground truth segmentations.  For the
case of a ground truth segmentation the appearance models are
normalized histograms of the pixel values within each region.  We see
that both the algebraic and spectral methods give good results,
leading to appearance models that are close to the ground truth.

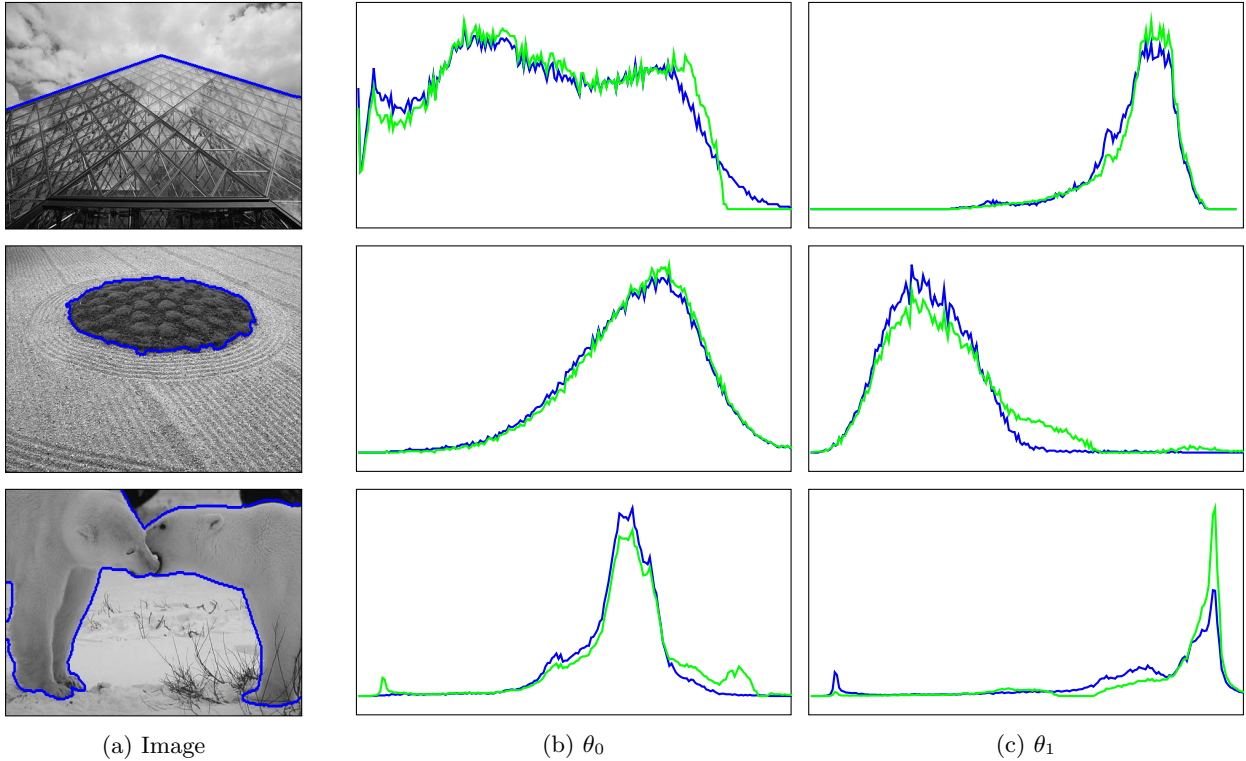
\begin{figure}[t]
    \centering
    
    \begin{subfigure}[t]{0.24\linewidth}
        \centering
        \includegraphics[width= .99\linewidth, height = 3cm,frame]{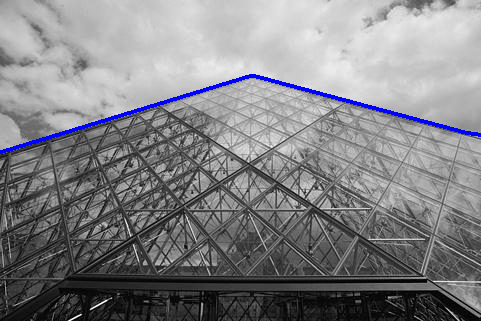}
    \end{subfigure}%
    \hspace{17pt}
    \begin{subfigure}[t]{0.35\linewidth}
        \centering
        \begin{tikzpicture}[trim axis left]
            \begin{axis}[
                scale only axis,
                height=3cm,
                width=\linewidth,
                max space between ticks=40,
                xmin=0, xmax=226,
                ticks=none,
                ylabel near ticks,
                legend style={font=\small, at={(0.5,0.95)},anchor=north},
                legend columns=-1
                ]
            \addplot+[mark=none, blue, thick] table[x index=0,y index=1] {data/estimation_comparison/BSD1_r24_theta.dat}; \label{line:gt1}
            \addplot+[mark=none, green, thick] table[x index=0,y index=6] {data/estimation_comparison/BSD1_r24_theta.dat};\label{line:alg}
            \end{axis}
        \end{tikzpicture}
    \end{subfigure}%
    ~
    \begin{subfigure}[t]{0.35\linewidth}
        \centering
        \begin{tikzpicture}[trim axis left]
            \begin{axis}[
                scale only axis,
                height=3cm,
                width=\linewidth,
                max space between ticks=40,
                xmin=0, xmax=250,
                ticks=none,
                ylabel near ticks,
                legend style={font=\small, at={(0.5,.95)},anchor=north},
                legend columns=-1
                ]
            \addplot+[mark=none, blue, thick] table[x index=0,y index=2] {data/estimation_comparison/BSD1_r24_theta.dat};
            \addplot+[mark=none, green, thick] table[x index=0,y index=5]  {data/estimation_comparison/BSD1_r24_theta.dat};
            \end{axis}
        \end{tikzpicture}
    \end{subfigure}
    
    \vspace{5pt}
    
    \begin{subfigure}[t]{0.24\linewidth}
        \centering
        \includegraphics[width= .99\linewidth,height = 3cm, frame]{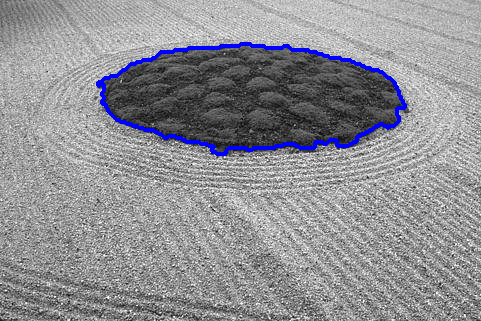}
    \end{subfigure}%
    \hspace{17pt}
    \begin{subfigure}[t]{0.35\linewidth}
        \centering
        \begin{tikzpicture}[trim axis left]
            \begin{axis}[
                scale only axis,
                height=3cm,
                width=\linewidth,
                max space between ticks=40,
                xmin=0, xmax=228,
                ticks=none,
                ylabel near ticks,
                legend style={font=\small, at={(0.5,0.25)},anchor=north},
                legend columns=-1
                ]
            \addplot+[mark=none, blue, thick] table[x index=0,y index=1] {data/estimation_comparison/BSD2_r24_theta.dat};
            \addplot+[mark=none, green, thick] table[x index=0,y index=6] {data/estimation_comparison/BSD2_r24_theta.dat};
            \end{axis}
        \end{tikzpicture}
    \end{subfigure}%
    ~
    \begin{subfigure}[t]{0.35\linewidth}
        \centering
        \begin{tikzpicture}[trim axis left]
            \begin{axis}[
                scale only axis,
                height=3cm,
                width=\linewidth,
                max space between ticks=40,
                xmin=0, xmax=228,
                ticks=none,
                ylabel near ticks,
                legend style={font=\small, at={(0.5,0.25)},anchor=north},
                legend columns=-1
                ]
            \addplot+[mark=none, blue, thick] table[x index=0,y index=2] {data/estimation_comparison/BSD2_r24_theta.dat};
            \addplot+[mark=none, green, thick] table[x index=0,y index=5]  {data/estimation_comparison/BSD2_r24_theta.dat};
            \end{axis}
        \end{tikzpicture}
    \end{subfigure}
    
    \vspace{5pt}
    
    \begin{subfigure}[t]{0.24\linewidth}
        \centering
        \includegraphics[width= .99\linewidth, height=3cm,frame]{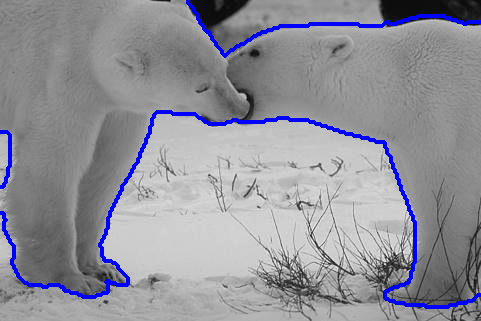}
        \caption{Image}
    \end{subfigure}%
    \hspace{17pt}
    \begin{subfigure}[t]{0.35\linewidth}
        \centering
        \begin{tikzpicture}[trim axis left]
            \begin{axis}[
                scale only axis,
                height=3cm,
                width=\linewidth,
                max space between ticks=40,
                xmin=0, xmax=200,
                ticks=none,
                ylabel near ticks,
                legend style={font=\small, at={(0.5,0.25)},anchor=north},
                legend columns=-1
                ]
            \addplot+[mark=none, blue, thick] table[x index=0,y index=1] {data/estimation_comparison/BSD5_r24_theta.dat};
            \addplot+[mark=none, green, thick] table[x index=0,y index=6] {data/estimation_comparison/BSD5_r24_theta.dat};
            \end{axis}
        \end{tikzpicture}
        \caption{$\theta_0$}
    \end{subfigure}%
    ~
    \begin{subfigure}[t]{0.35\linewidth}
        \centering
        \begin{tikzpicture}[trim axis left]
            \begin{axis}[
                scale only axis,
                height=3cm,
                width=\linewidth,
                max space between ticks=40,
                xmin=0, xmax=200,
                ticks=none,
                ylabel near ticks,
                legend style={font=\small, at={(0.5,0.25)},anchor=north},
                legend columns=-1
                ]
            \addplot+[mark=none, blue, thick] table[x index=0,y index=2] {data/estimation_comparison/BSD5_r24_theta.dat};
            \addplot+[mark=none, green, thick] table[x index=0,y index=5]  {data/estimation_comparison/BSD5_r24_theta.dat};
            \end{axis}
        \end{tikzpicture}
        \caption{$\theta_1$}
    \end{subfigure}
    \caption{Estimation of appearance models with the algebraic
      method.  In (a) we show the input images and their ground truth
      segmentation.  In (b) and (c) we show the appearance models
      computed using the ground truth segmentation in blue
      (\ref{line:gt1}) and the algebraic method in green
      (\ref{line:alg}).  The images are from the Berkeley Segmentation
      \cite{martin01} dataset.}
    \label{fig:example_estimation_algebraic}
\end{figure}

\begin{figure}[t]
    \centering
    
    \begin{subfigure}[t]{0.24\linewidth}
        \centering
        \includegraphics[width= .99\linewidth, height = 3cm,frame]{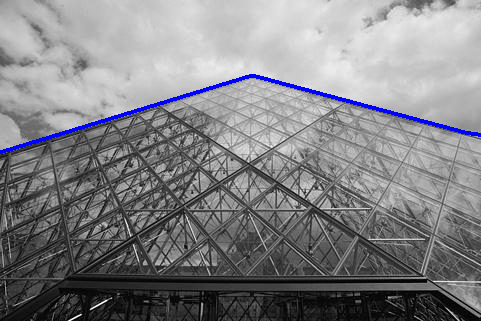}
    \end{subfigure}%
    \hspace{17pt}
    \begin{subfigure}[t]{0.35\linewidth}
        \centering
        \begin{tikzpicture}[trim axis left]
            \begin{axis}[
                scale only axis,
                height=3cm,
                width=\linewidth,
                max space between ticks=40,
                xmin=0, xmax=226,
                ticks=none,
                ylabel near ticks,
                legend style={font=\small, at={(0.5,0.95)},anchor=north},
                legend columns=-1
                ]
            \addplot+[mark=none, blue, thick] table[x index=0,y index=1] {data/estimation_comparison/BSD1_r24_theta.dat}; \label{line:gt2}
            \addplot+[mark=none, red, thick] table[x index=0,y index=4] {data/estimation_comparison/BSD1_r24_theta.dat};\label{line:spec}
            \end{axis}
        \end{tikzpicture}
    \end{subfigure}%
    ~
    \begin{subfigure}[t]{0.35\linewidth}
        \centering
        \begin{tikzpicture}[trim axis left]
            \begin{axis}[
                scale only axis,
                height=3cm,
                width=\linewidth,
                max space between ticks=40,
                xmin=0, xmax=250,
                ticks=none,
                ylabel near ticks,
                legend style={font=\small, at={(0.5,.95)},anchor=north},
                legend columns=-1
                ]
            \addplot+[mark=none, red, thick] table[x index=0,y index=3] {data/estimation_comparison/BSD1_r24_theta.dat};
            \addplot+[mark=none, blue, thick] table[x index=0,y index=2] {data/estimation_comparison/BSD1_r24_theta.dat};
            \end{axis}
        \end{tikzpicture}
    \end{subfigure}
    
    \vspace{5pt}
    
    \begin{subfigure}[t]{0.24\linewidth}
        \centering
        \includegraphics[width= .99\linewidth,height = 3cm, frame]{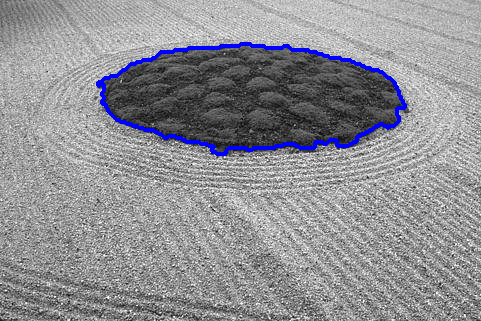}
    \end{subfigure}%
    \hspace{17pt}
    \begin{subfigure}[t]{0.35\linewidth}
        \centering
        \begin{tikzpicture}[trim axis left]
            \begin{axis}[
                scale only axis,
                height=3cm,
                width=\linewidth,
                max space between ticks=40,
                xmin=0, xmax=228,
                ticks=none,
                ylabel near ticks,
                legend style={font=\small, at={(0.5,0.25)},anchor=north},
                legend columns=-1
                ]
            \addplot+[mark=none, blue, thick] table[x index=0,y index=1] {data/estimation_comparison/BSD2_r24_theta.dat};
            \addplot+[mark=none, red, thick] table[x index=0,y index=4] {data/estimation_comparison/BSD2_r24_theta.dat};            
            \end{axis}
        \end{tikzpicture}
    \end{subfigure}%
    ~
    \begin{subfigure}[t]{0.35\linewidth}
        \centering
        \begin{tikzpicture}[trim axis left]
            \begin{axis}[
                scale only axis,
                height=3cm,
                width=\linewidth,
                max space between ticks=40,
                xmin=0, xmax=228,
                ticks=none,
                ylabel near ticks,
                legend style={font=\small, at={(0.5,0.25)},anchor=north},
                legend columns=-1
                ]
            \addplot+[mark=none, blue, thick] table[x index=0,y index=2] {data/estimation_comparison/BSD2_r24_theta.dat};
            \addplot+[mark=none, red, thick] table[x index=0,y index=3] {data/estimation_comparison/BSD2_r24_theta.dat};
            \end{axis}
        \end{tikzpicture}
    \end{subfigure}
    
    \vspace{5pt}
    
    \begin{subfigure}[t]{0.24\linewidth}
        \centering
        \includegraphics[width= .99\linewidth, height=3cm,frame]{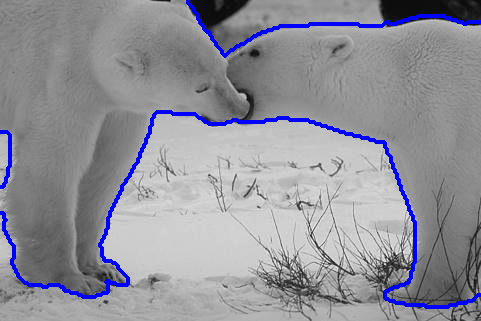}
        \caption{Image}
    \end{subfigure}%
    \hspace{17pt}
    \begin{subfigure}[t]{0.35\linewidth}
        \centering
        \begin{tikzpicture}[trim axis left]
            \begin{axis}[
                scale only axis,
                height=3cm,
                width=\linewidth,
                max space between ticks=40,
                xmin=0, xmax=200,
                ticks=none,
                ylabel near ticks,
                legend style={font=\small, at={(0.5,0.25)},anchor=north},
                legend columns=-1
                ]
            \addplot+[mark=none, blue, thick] table[x index=0,y index=4] {data/estimation_comparison/BSD5_r24_theta.dat};
            \addplot+[mark=none, red, thick] table[x index=0,y index=1] {data/estimation_comparison/BSD5_r24_theta.dat};            
            \end{axis}
        \end{tikzpicture}
        \caption{$\theta_0$}
    \end{subfigure}%
    ~
    \begin{subfigure}[t]{0.35\linewidth}
        \centering
        \begin{tikzpicture}[trim axis left]
            \begin{axis}[
                scale only axis,
                height=3cm,
                width=\linewidth,
                max space between ticks=40,
                xmin=0, xmax=200,
                ticks=none,
                ylabel near ticks,
                legend style={font=\small, at={(0.5,0.25)},anchor=north},
                legend columns=-1
                ]
            \addplot+[mark=none, blue, thick] table[x index=0,y index=3] {data/estimation_comparison/BSD5_r24_theta.dat};
            \addplot+[mark=none, red, thick] table[x index=0,y index=2] {data/estimation_comparison/BSD5_r24_theta.dat};            
            \end{axis}
        \end{tikzpicture}
        \caption{$\theta_1$}
    \end{subfigure}
    \caption{Estimation of appearance models with the spectral method.
      In (a) we show the input images and their ground truth
      segmentation.  In (b) and (c) we show the appearance models
      computed using the ground truth segmentation in blue
      (\ref{line:gt2}) and the spectral method in red
      (\ref{line:spec}). The images are from the Berkeley Segmentation
      \cite{martin01} dataset.}
    \label{fig:example_estimation_spectral}
\end{figure}

\section{Numerical Experiments}
\label{sec:experiments}

All of our algorithms were implemented in Matlab and the experiments
presented here were run on a Laptop computer with an Intel Core
i5-6200U CPU 2.30GHz with 8 Gb of RAM.

\subsection{Evaluation Measures}
\label{sec:eval}

We use the Bhattacharyya distance to measure similarity between two
probability distributions.  Let $p$ and $q$ be two distributions over
a finite set $Z$.  The Bhattacharyya distance beetween $p$ and $q$ is,
\begin{equation}
    d_{B}(p, q) =  -\ln\left( \sum_{x \in Z} \sqrt{p(x)q(x)}\right).
\end{equation}

To evaluate the quality of the appearance models we estimate we
compare them to the appearance models defined by a ground truth
segmentation.

Let $I$ be an image with a ground truth segmentation defined by $\R_0$
and $\R_1$.  Let $\theta_0$ and $\theta_1$ be the normalized
histograms of the pixel values within each region.  Let
$\hat{\theta}_0$ and $\hat{\theta}_1$ be the appearance models
estimated from $I$ using one of our algorithms.  We assess the quality
of the estimated appearance models using a sum of two Bhattacharyya
distances, allowing for a permutation of the region labels,
\begin{equation}
    D_{B} = \min\left(\frac{d_{B}(\theta_0, \hat{\theta}_0) +d_{B}(\theta_1, \hat{\theta}_1)}{2}, \frac{d_{B}(\theta_0, \hat{\theta}_1) + d_{B}(\theta_1, \hat{\theta}_0)}{2}\right).
\end{equation}

We also evaluate the accuracy of segmentations obtained using
different methods by comparing them to the ground truth segmentation.

We assess the overlap between two regions $J,Q
\subseteq \Omega$ in different segmentations using the Jaccard
index, $$J(S,Q) = |S \cap Q|/|S \cup Q|.$$

Let $\Q_0$ and $\Q_1$ be two regions obtained by segmenting $I$.  We
compare the segmentation defined by $\Q_0$ and $\Q_1$ to the ground
truth segmentation defined by $\R_0$ and $\R_1$ using a sum of two
Jaccard indices, again allowing for a permutation of the region
labels,
\begin{equation}
\jac = \max\left(\frac{J(\R_0,\Q_0)+J(\R_1,\Q_1)}{2},\frac{J(\R_0,\Q_1)+J(\R_1,\Q_0)}{2}\right).
\end{equation}

\subsection{Synthetic Data}
\label{sec:synt}

\begin{figure}[t]
    \centering
    \includegraphics[width=0.19\linewidth]{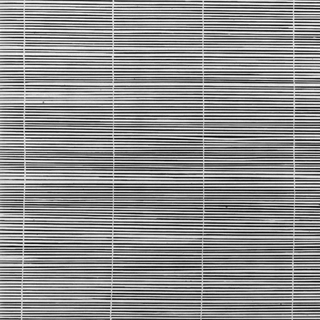}
    \includegraphics[width=0.19\linewidth]{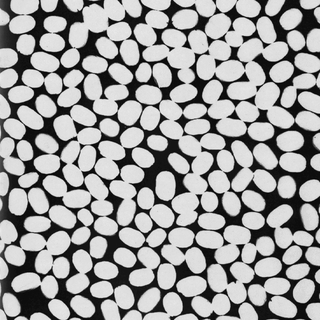}
    \includegraphics[width=0.19\linewidth]{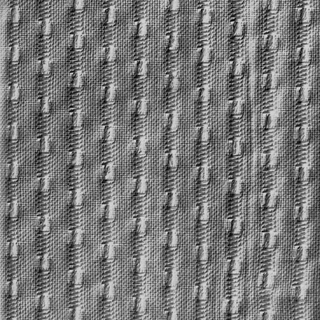}
    \includegraphics[width=0.19\linewidth]{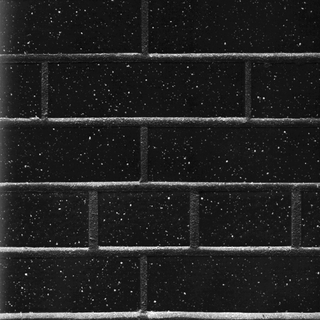}
    \includegraphics[width=0.19\linewidth]{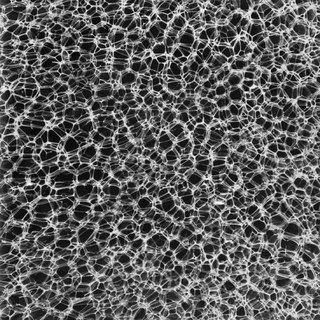}
    \vspace{.1cm}
    
    \includegraphics[width=0.19\linewidth]{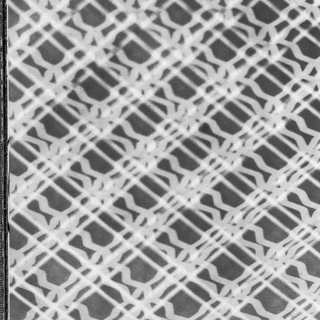}
    \includegraphics[width=0.19\linewidth]{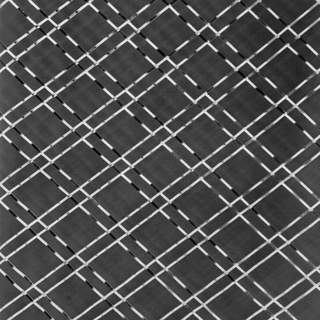}
    \includegraphics[width=0.19\linewidth]{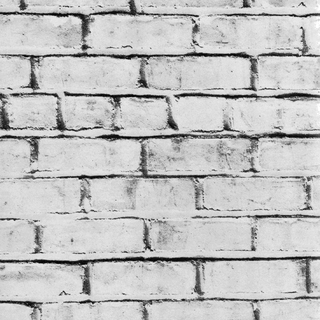}
    \includegraphics[width=0.19\linewidth]{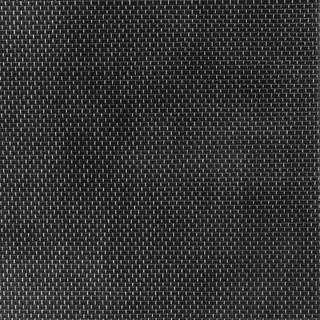}
    \includegraphics[width=0.19\linewidth]{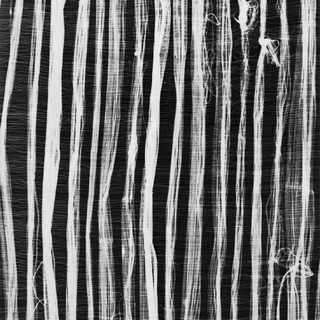}
    
    \caption{Selected Brodatz patterns}
    \label{fig:brodatz_patt}
\end{figure}

\begin{figure}[t]
    \centering
    \begin{subfigure}[t]{0.19\linewidth}
        \includegraphics[width=.99\linewidth, frame]{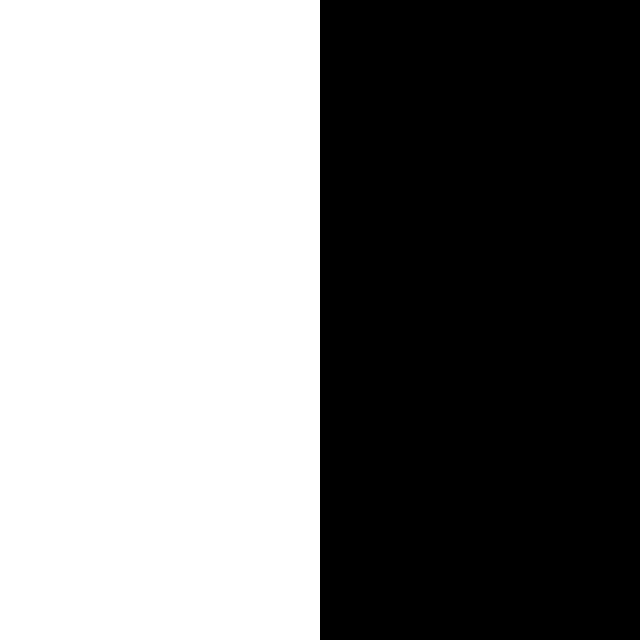}
        \caption{GT1}
    \end{subfigure}%
    $\,\,$%
    \begin{subfigure}[t]{0.19\linewidth}
        \includegraphics[width=.99\linewidth, frame]{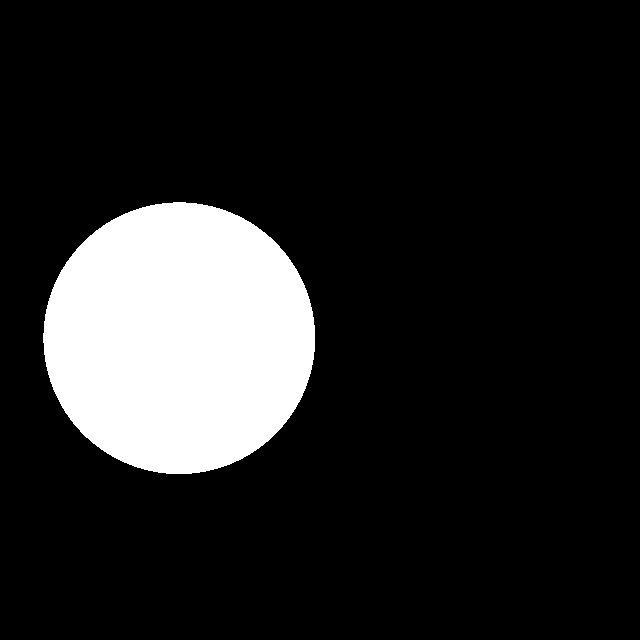}
        \caption{GT2}
    \end{subfigure}%
    $\,\,$%
    \begin{subfigure}[t]{0.19\linewidth}
        \includegraphics[width=.99\linewidth, frame]{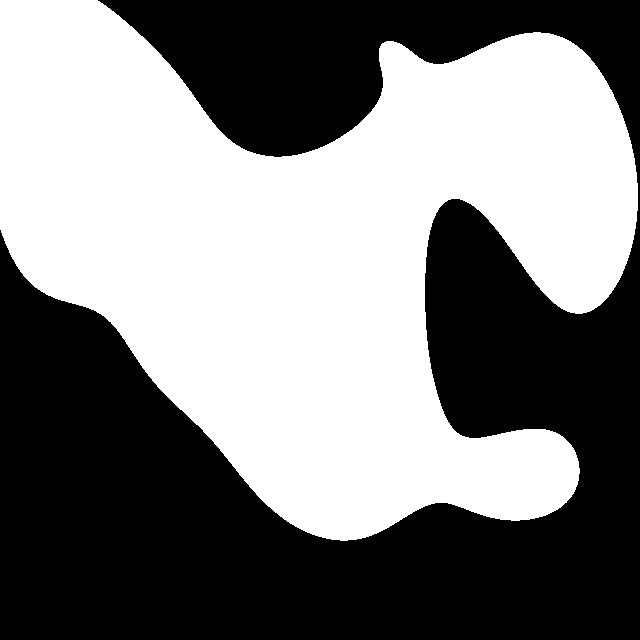}
        \caption{GT3}
    \end{subfigure}%
    $\,\,$%
    \begin{subfigure}[t]{0.19\linewidth}
        \includegraphics[width=.99\linewidth, frame]{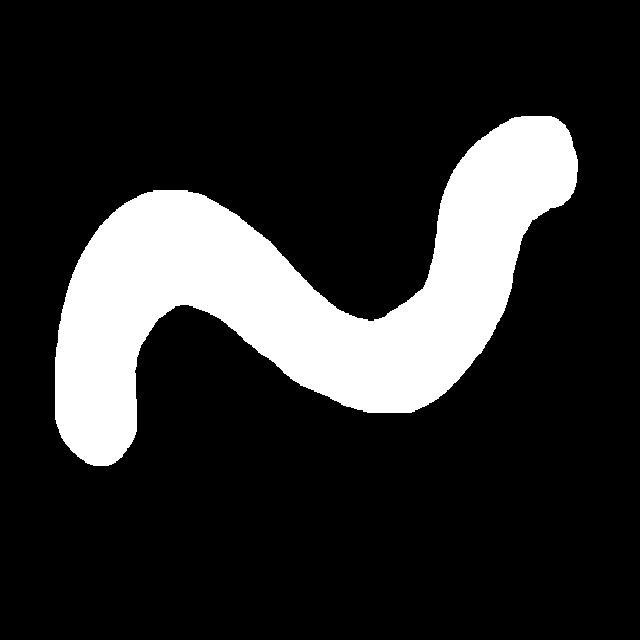}
        \caption{GT4}
    \end{subfigure}%
    $\,\,$%
    \begin{subfigure}[t]{0.19\linewidth}
        \includegraphics[width=.99\linewidth, frame]{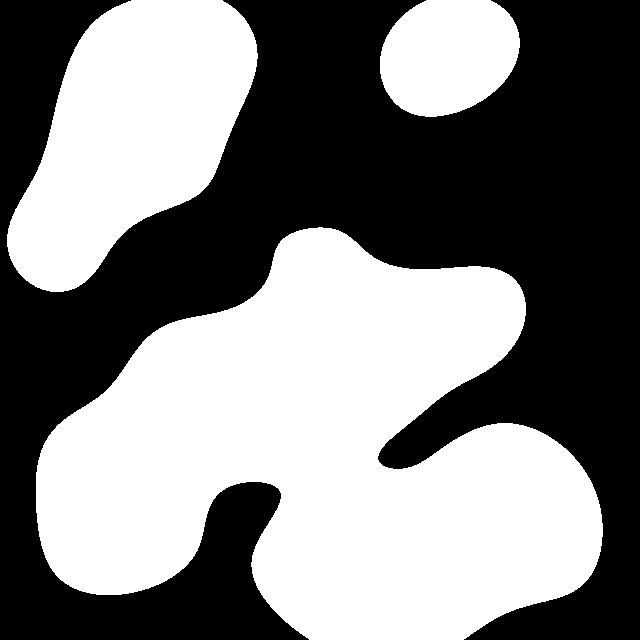}
        \caption{GT5}
    \end{subfigure}%
    
    \caption{Ground truth segmentations used to generate synthetic data.}

    \label{fig:gts}
\end{figure}

\begin{figure}[t]
    \centering
    \begin{subfigure}{\linewidth}
    \centerline{
        \includegraphics[width=.19\linewidth]{images/brodatz_patterns/BR_2_s.png}
        \includegraphics[width=.19\linewidth]{images/brodatz_patterns/BR_8_s.png}}
        \caption{Textures}
    \end{subfigure}
    \vspace{.2cm}

    \begin{subfigure}{\linewidth}
        \includegraphics[width=.19\linewidth]{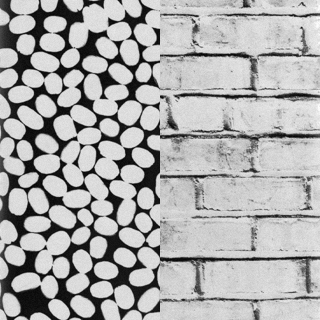}
        \includegraphics[width=.19\linewidth]{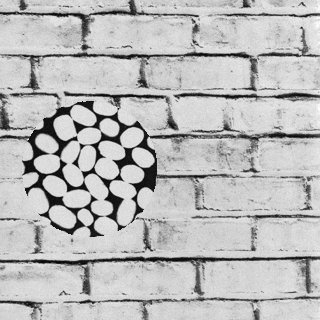}
        \includegraphics[width=.19\linewidth]{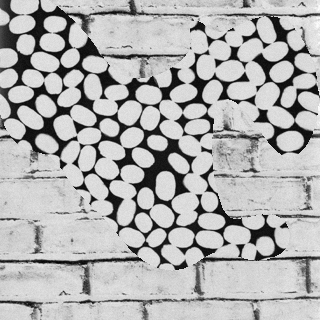}
        \includegraphics[width=.19\linewidth]{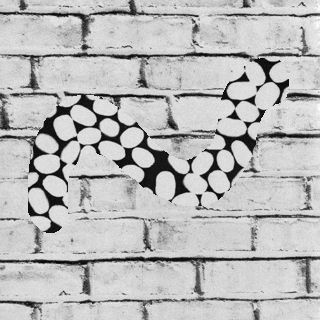}
        \includegraphics[width=.19\linewidth]{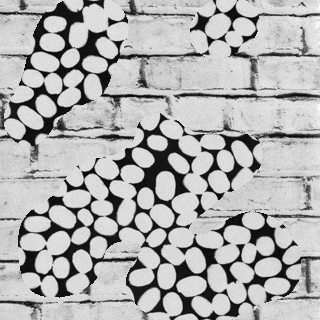}
        \caption{Synthetic images}
    \end{subfigure}
    \caption{Examples of synthesized images with textures.}
    \label{fig:syntheticdata}
\end{figure}

We first illustrate the results of a series of experiments with
synthetic data.  To generate the synthetic data we used the
segmentations masks in Figure~\ref{fig:gts} together with pairs of
images defined as follows:
\begin{itemize}
    \item IID: we used 50 pairs of random appearance models to
      generate pairs of images.  For each appearance model we generate
      a $320 \times 320$ image where the pixel values are independent
      samples from the corresponding distribution.

    \item Brodatz: we selected all possible pairings of images from
      the Brodatz textures \cite{brodatz1966textures} shown in
      Figure~\ref{fig:brodatz_patt}.  We resized the images to be $320
      \times 320$ pixels and added uniform IID noise to the pixels to
      to remove quantization artifacts.
\end{itemize}
For each pair of images defined above we use the segmentation masks in
Figure~\ref{fig:gts} to generate graylevel images with two regions.

Given a pair of images $(A,B)$ and a binary segmentation mask $M$
we generate a graylevel image $I$ with $I(x) = A(x)M(x)+B(x)(1-M(x))$.
Figure~\ref{fig:syntheticdata} shows the images generated using two
Brodatz patterns.

\subsection{Evaluating the effect of $r$}

In our experiments we set $r = \rho \sqrt{\Omega}$ where $\rho$ is a
parameter set by the user.  This makes the distance
$r$ adaptive to the image resolution.

In Figure~\ref{fig:example_estimation_rs} we evaluate the quality of
the appearance models estimated by our methods using different values
of $\rho$ on the synthetic data defined by the ground truth
segmentations GT1 and GT2 (Figure \ref{fig:gts}).  For these
experiments we set $w_0$, $w_1$ and $\epsilon$ using the ground truth
values defined by the corresponding segmentation masks.

Both of our algorithms almost perfectly recover the underlying
appearance models for images where the pixels in each region are IID.
In this case the methods work well over the whole range of values of
$\rho$ tested. This is expected since these images strictly follow
both Assumption \ref{as:homogeneity} and Assumption
\ref{as:independence} and, therefore, provide the optimal setting for
our algorithms.

For images with Brodatz textures Assumption \ref{as:independence} is
violated for small values of $\rho$.  As $\rho$ increases the
assumption is satisfied and the quality of our estimation improves.

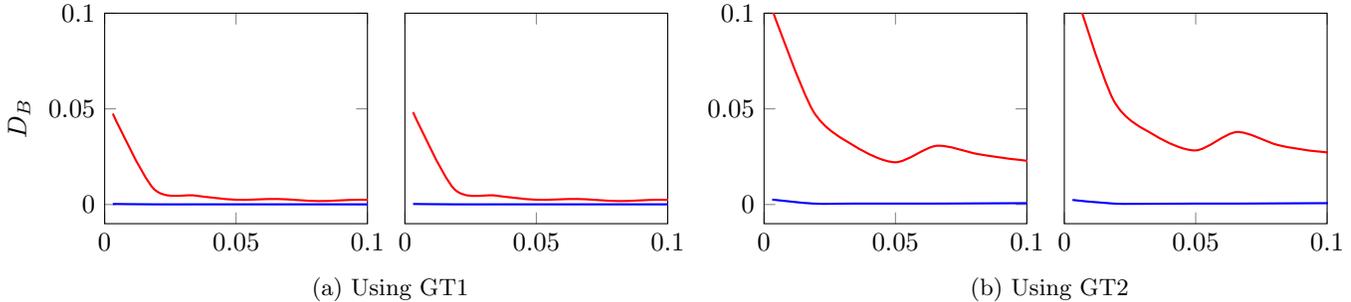
\begin{figure}[t]
    \centering    
    \begin{subfigure}[t]{0.46\linewidth}
        \centering
        \begin{tikzpicture}[trim axis left]
            \begin{axis}[
                scale only axis,
                height=2.8cm,
                width=.46\linewidth,
                ylabel={$D_B$},
                xmin=0.0, xmax=0.1,
                ymin=-0.01, ymax=0.1,
                xtick distance={.05},
                ylabel near ticks,
                scaled y ticks=false,
                tick label style={/pgf/number format/fixed},
                legend style={font=\small, at={(0.5,0.25)},anchor=north},
                legend columns=-1
                ]
            \addplot+[smooth, mark=none, blue, thick] table[x index=0,y index=3] {data/estimation_rs_gt1.dat};
            \addplot+[smooth, mark=none, red, thick] table[x index=0,y index=4] {data/estimation_rs_gt1.dat};
            \end{axis}
        \end{tikzpicture}%
        ~
        \begin{tikzpicture}[trim axis left]
            \begin{axis}[
                scale only axis,
                height=2.8cm,
                width=.46\linewidth,
                xmin=0.0, xmax=0.1,
                ymin=-0.01, ymax=0.1,
                xtick distance={.05},
                ymajorticks=false,
                ylabel near ticks,
                tick label style={/pgf/number format/fixed},
                legend style={font=\small, at={(0.5,0.25)},anchor=north},
                legend columns=-1
                ]
            \addplot+[smooth,mark=none, blue, thick] table[x index=0,y index=1] {data/estimation_rs_gt1.dat};
            \addplot+[smooth,mark=none, red, thick] table[x index=0,y index=2] {data/estimation_rs_gt1.dat};
            \end{axis}
        \end{tikzpicture}
        \caption{Using GT1}
    \end{subfigure}%
    \hspace{30pt}
    \begin{subfigure}[t]{0.46\linewidth}
        \centering
        \begin{tikzpicture}[trim axis left]
            \begin{axis}[
                scale only axis,
                height=2.8cm,
                width=.46\linewidth,
                xmin=0.0, xmax=0.1,
                ymin=-0.01, ymax=0.1,
                xtick distance={.05},
                ylabel near ticks,
                scaled y ticks=false,
                tick label style={/pgf/number format/fixed},
                legend style={font=\small, at={(0.5,0.25)},anchor=north},
                legend columns=-1
                ]
            \addplot+[smooth,mark=none, blue, thick] table[x index=0,y index=3] {data/estimation_rs_gt2.dat}; \label{line:iid}
            \addplot+[smooth,mark=none, red, thick] table[x index=0,y index=4] {data/estimation_rs_gt2.dat}; \label{line:B1}
            \end{axis}
        \end{tikzpicture}%
        ~
        \begin{tikzpicture}[trim axis left]
            \begin{axis}[
                scale only axis,
                height=2.8cm,
                width=.46\linewidth,
                xmin=0.0, xmax=0.1,
                ymin=-0.01, ymax=0.1,
                xtick distance={.05},
                ymajorticks=false,
                ylabel near ticks,
                tick label style={/pgf/number format/fixed},
                legend style={font=\small, at={(0.5,0.25)},anchor=north},
                legend columns=-1
                ]
            \addplot+[smooth,mark=none, blue, thick] table[x index=0,y index=1] {data/estimation_rs_gt2.dat};
            \addplot+[smooth,mark=none, red, thick] table[x index=0,y index=2] {data/estimation_rs_gt2.dat};
            \end{axis}
        \end{tikzpicture}
        \caption{Using GT2}
    \end{subfigure}
    \caption{Average appearance model estimation error ($D_B$) as a
      function of $\rho$ on images composed of IID (\ref{line:iid})
      and Brodatz (\ref{line:B1}) patterns disposed as in GT1 and GT2.
      For both (a) and (b) the results on the left are from the
      algebraic method, whereas the results on the right are from the
      spectral method.}
      \label{fig:example_estimation_rs}
\end{figure}

\subsection{Appearance Model Evaluation on Synthetic Images}

We compare the performance of our methods for estimating appearance
models to a variation of the iterative scheme described in
\cite{tang2014pseudo}, here called ALT.

In ALT, we start with an initial segmentation of the image and
alternate between computing new appearance models using the current
segmentation and computing a new segmentation using the current
appearance models.  This procedure is iterated until convergence.  To
update the appearance models using the current segmentation we
histogram the pixel values in each region.  We ``smooth'' the
histograms by adding a constant $K = 1$ to their bins before
normalizing them.  To update the segmentation using the current
appearance models we minimize the energy in Equation (\ref{eq:E_seg})
using a max-flow/min-cut algorithm
(\cite{greig1989exact,boykov1999fast}).

For the experiments described here the initial segmentation used for
ALT is defined by a square region in the middle of the image.  Figure
\ref{fig:alt_estimation} shows an example of how the segmentation and
appearance models evolve over time.  Empirically, we found that ALT
works well in many examples but a typical failure mode leads to
assigning the whole image to single segment.

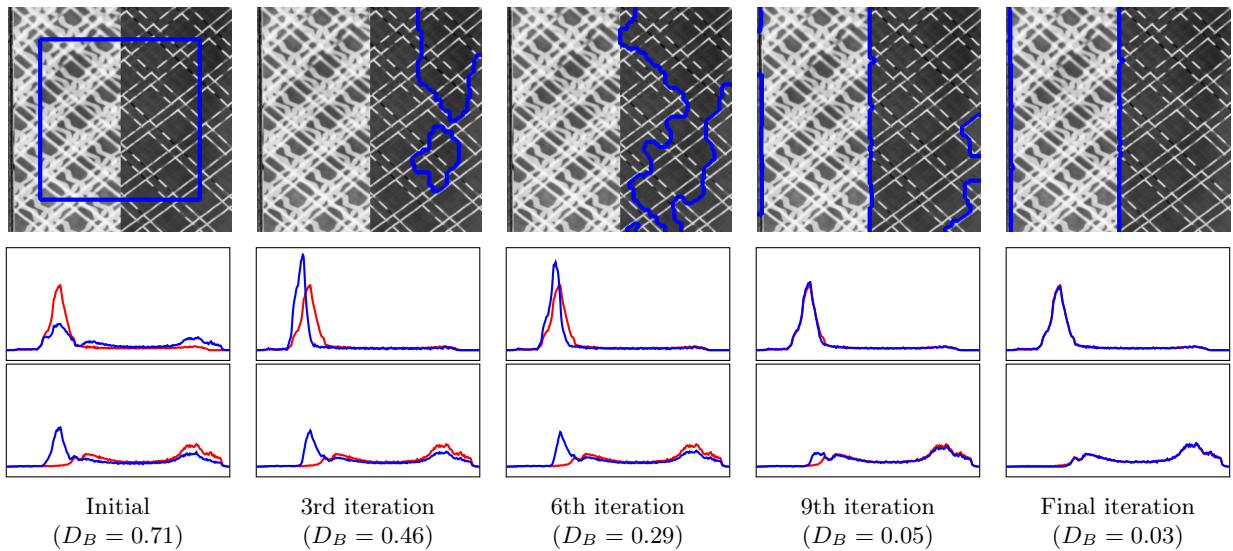
\begin{figure}[t]
    \centering
    \hspace{1pt}
    \begin{subfigure}[t]{0.18\linewidth}
        \includegraphics[width=\linewidth]{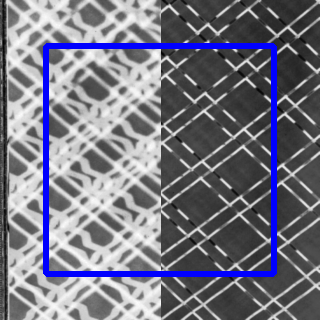}
    \end{subfigure}%
    ~~
    \begin{subfigure}[t]{0.18\linewidth}
        \includegraphics[width=\linewidth]{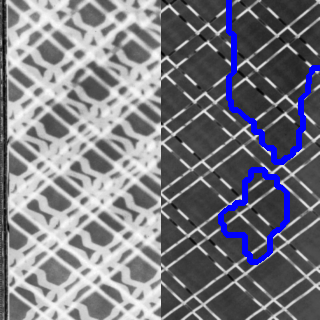}
    \end{subfigure}%
    ~~   
    \begin{subfigure}[t]{0.18\linewidth}
        \includegraphics[width=\linewidth]{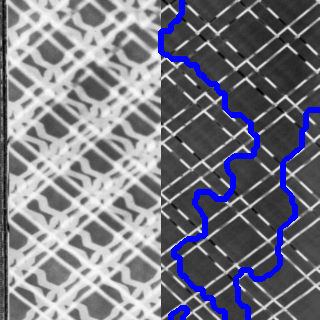}
    \end{subfigure}%
    ~~
    \begin{subfigure}[t]{0.18\linewidth}
        \includegraphics[width=\linewidth]{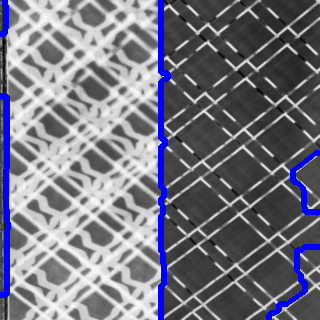}
    \end{subfigure}%
    ~~
    \begin{subfigure}[t]{0.18\linewidth}
        \includegraphics[width=\linewidth]{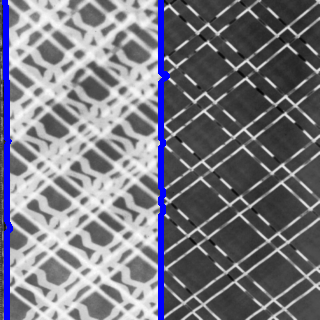}
    \end{subfigure}
    \vspace{5pt}
    
    \begin{subfigure}[t]{0.18\linewidth}
        \centering
        \begin{tikzpicture}[trim axis left]
            \begin{axis}[
                scale only axis,
                height=1.5cm,
                width=\linewidth,
                max space between ticks=40,
                xmin=1, xmax=241,
                ymax=.065,
                ticks=none,
                ylabel near ticks,
                legend style={font=\small, at={(0.5,0.25)},anchor=north},
                legend columns=-1
                ]            
            \addplot+[smooth,mark=none, red,  thick] table[x index=0,y index=2] {data/alt_evolution/thetas_6-7.dat}; 
            \addplot+[smooth,mark=none, blue,  thick] table[x index=0,y index=3] {data/alt_evolution/thetas_6-7.dat}; 
            \end{axis}
        \end{tikzpicture}
    \end{subfigure}%
    ~~
    \begin{subfigure}[t]{0.18\linewidth}
        \centering
        \begin{tikzpicture}[trim axis left]
            \begin{axis}[
                scale only axis,
                height=1.5cm,
                width=\linewidth,
                max space between ticks=40,
                xmin=1, xmax=241,
                ymax=.065,
                ticks=none,
                ylabel near ticks,
                legend style={font=\small, at={(0.5,0.25)},anchor=north},
                legend columns=-1
                ]
            \addplot+[smooth,mark=none, red,  thick] table[x index=0,y index=2] {data/alt_evolution/thetas_6-7.dat}; 
            \addplot+[smooth,mark=none, blue,  thick] table[x index=0,y index=7] {data/alt_evolution/thetas_6-7.dat};
            \end{axis}
        \end{tikzpicture}
        \captionsetup{justification=centering}
    \end{subfigure}%
    ~~
    \begin{subfigure}[t]{0.18\linewidth}
        \centering
        \begin{tikzpicture}[trim axis left]
            \begin{axis}[
                scale only axis,
                height=1.5cm,
                width=\linewidth,
                max space between ticks=40,
                xmin=1, xmax=241,
                ymax=.065,
                ticks=none,
                ylabel near ticks,
                legend style={font=\small, at={(0.5,0.25)},anchor=north},
                legend columns=-1
                ]
            \addplot+[smooth,mark=none, red,  thick] table[x index=0,y index=2] {data/alt_evolution/thetas_6-7.dat}; 
            \addplot+[smooth,mark=none, blue,  thick] table[x index=0,y index=13] {data/alt_evolution/thetas_6-7.dat};  
            \end{axis}
        \end{tikzpicture}
        \captionsetup{justification=centering}
    \end{subfigure}%
    ~~
    \begin{subfigure}[t]{0.18\linewidth}
        \centering
        \begin{tikzpicture}[trim axis left]
            \begin{axis}[
                scale only axis,
                height=1.5cm,
                width=\linewidth,
                max space between ticks=40,
                xmin=1, xmax=241,
                ymax=.065,
                ticks=none,
                ylabel near ticks,
                legend style={font=\small, at={(0.5,0.25)},anchor=north},
                legend columns=-1
                ]
            \addplot+[smooth,mark=none, red,  thick] table[x index=0,y index=2] {data/alt_evolution/thetas_6-7.dat}; 
            \addplot+[smooth,mark=none, blue,  thick] table[x index=0,y index=19] {data/alt_evolution/thetas_6-7.dat};  
            \end{axis}
        \end{tikzpicture}
        \captionsetup{justification=centering}
    \end{subfigure}%
    ~~
    \begin{subfigure}[t]{0.18\linewidth}
        \centering
        \begin{tikzpicture}[trim axis left]
            \begin{axis}[
                scale only axis,
                height=1.5cm,
                width=\linewidth,
                max space between ticks=40,
                xmin=1, xmax=241,
                ymax=.065,
                ticks=none,
                ylabel near ticks,
                legend style={font=\small, at={(0.5,0.25)},anchor=north},
                legend columns=-1
                ]
            \addplot+[smooth,mark=none, red,  thick] table[x index=0,y index=2] {data/alt_evolution/thetas_6-7.dat};
            \addplot+[smooth,mark=none, blue,  thick] table[x index=0,y index=21] {data/alt_evolution/thetas_6-7.dat};
            \end{axis}
        \end{tikzpicture}
        \captionsetup{justification=centering}
    \end{subfigure}
    
    \begin{subfigure}[t]{0.18\linewidth}
        \centering
        \begin{tikzpicture}[trim axis left]
            \begin{axis}[
                scale only axis,
                height=1.5cm,
                width=\linewidth,
                max space between ticks=40,
                xmin=1, xmax=241,
                ymax=.065,
                ticks=none,
                ylabel near ticks,
                legend style={font=\small, at={(0.5,0.25)},anchor=north},
                legend columns=-1
                ]
            \addplot+[smooth,mark=none, red,  thick] table[x index=0,y index=1] {data/alt_evolution/thetas_6-7.dat}; 
            \addplot+[smooth,mark=none, blue,  thick] table[x index=0,y index=4] {data/alt_evolution/thetas_6-7.dat};  
            \end{axis}
        \end{tikzpicture}
        \caption*{Initial \\ ($D_B = 0.71$)}
    \end{subfigure}%
    ~~
    \begin{subfigure}[t]{0.18\linewidth}
        \centering
        \begin{tikzpicture}[trim axis left]
            \begin{axis}[
                scale only axis,
                height=1.5cm,
                width=\linewidth,
                max space between ticks=40,
                xmin=1, xmax=241,
                ymax=.065,
                ticks=none,
                ylabel near ticks,
                legend style={font=\small, at={(0.5,0.25)},anchor=north},
                legend columns=-1
                ]
            \addplot+[smooth,mark=none, red,  thick] table[x index=0,y index=1] {data/alt_evolution/thetas_6-7.dat}; 
            \addplot+[smooth,mark=none, blue,  thick] table[x index=0,y index=8] {data/alt_evolution/thetas_6-7.dat};  
            \end{axis}
        \end{tikzpicture}
        \captionsetup{justification=centering}
        \caption*{3rd iteration \\ ($D_B = 0.46$)}
    \end{subfigure}%
    ~~
    \begin{subfigure}[t]{0.18\linewidth}
        \centering
        \begin{tikzpicture}[trim axis left]
            \begin{axis}[
                scale only axis,
                height=1.5cm,
                width=\linewidth,
                max space between ticks=40,
                xmin=1, xmax=241,
                ymax=.065,
                ticks=none,
                ylabel near ticks,
                legend style={font=\small, at={(0.5,0.25)},anchor=north},
                legend columns=-1
                ]
            \addplot+[smooth,mark=none, red,  thick] table[x index=0,y index=1] {data/alt_evolution/thetas_6-7.dat}; 
            \addplot+[smooth,mark=none, blue,  thick] table[x index=0,y index=12] {data/alt_evolution/thetas_6-7.dat};
            \end{axis}
        \end{tikzpicture}
        \captionsetup{justification=centering}
        \caption*{6th iteration \\ ($D_B = 0.29$)}
    \end{subfigure}%
    ~~
    \begin{subfigure}[t]{0.18\linewidth}
        \centering
        \begin{tikzpicture}[trim axis left]
            \begin{axis}[
                scale only axis,
                height=1.5cm,
                width=\linewidth,
                max space between ticks=40,
                xmin=1, xmax=241,
                ymax=.065,
                ticks=none,
                ylabel near ticks,
                legend style={font=\small, at={(0.5,0.25)},anchor=north},
                legend columns=-1
                ]
            \addplot+[smooth,mark=none, red,  thick] table[x index=0,y index=1] {data/alt_evolution/thetas_6-7.dat}; 
            \addplot+[smooth,mark=none, blue,  thick] table[x index=0,y index=18] {data/alt_evolution/thetas_6-7.dat};
            \end{axis}
        \end{tikzpicture}
        \captionsetup{justification=centering}
        \caption*{9th iteration \\ ($D_B = 0.05$)}
    \end{subfigure}%
    ~~
    \begin{subfigure}[t]{0.18\linewidth}
        \centering
        \begin{tikzpicture}[trim axis left]
            \begin{axis}[
                scale only axis,
                height=1.5cm,
                width=\linewidth,
                max space between ticks=40,
                xmin=1, xmax=241,
                ymax=.065,
                ticks=none,
                ylabel near ticks,
                legend style={font=\small, at={(0.5,0.25)},anchor=north},
                legend columns=-1
                ]
            \addplot+[smooth,mark=none, red,  thick] table[x index=0,y index=1] {data/alt_evolution/thetas_6-7.dat}; \label{line:dist_gt}
            \addplot+[smooth,mark=none, blue,  thick] table[x index=0,y index=22] {data/alt_evolution/thetas_6-7.dat}; \label{line:dist_est}
            
            \end{axis}
        \end{tikzpicture}
        \captionsetup{justification=centering}
        \caption*{Final iteration \\ ($D_B = 0.03$)}
    \end{subfigure}%
    
    \caption{Model estimation and segmentation using ALT. On the
      bottom, we see how both foreground and background color
      distributions estimated by ALT (\ref{line:dist_est}) evolve
      compared to the ground truth appearance models
      (\ref{line:dist_gt}). The evolution of the segmentations given
      the models is shown on top.}
    \label{fig:alt_estimation}
\end{figure}

Table~\ref{tab:quantitative_results_estimation} compares the results
of our methods to the result of ALT using several values of $\lambda$
for the segmentation step.  We used $\rho = 0.06$ (which corresponds
to $r \approx 20$ pixels for the $320 \times 320$ synthetic images)
for both the algebraic and spectral methods.  We evaluated our
algorithms using three different approaches for selecting $w_0$, $w_1$
and $\epsilon$.  In the first approach we set the parameters to the
their ground truth values defined by the corresponding segmentation
mask.  In the second approach we fix the parameters to typical values
that work well for many images.  In the third approach we search over
the parameters explicitly (see Section~\ref{sec:w0w1e}).  All of the
approaches lead to good results but searching for the optimal
parameters leads to a significant increase in runtime for the
algebraic method.

We see that our algorithms perform extremely well on images where the
pixel values in each region are IID.  The results on images with
textures are also good and compare favorably to ALT. This result is
compelling in particular because the proposed methods do not rely on
an iterative model re-estimation scheme such as in ALT, which makes
them faster and independent of initialization.  The average runtime of
the different methods are shown in the last column of Table
\ref{tab:quantitative_results_estimation}.

\begin{table}
\setlength{\fboxrule}{.5pt}
\setlength{\fboxsep}{0pt}
\centering
\caption{Average $D_B$ distance (lower is better) between estimated and ground truth
  appearance models on the synthetic data generated using different
  segmentation masks.  We evaluate our algorithms using different
  methods for selecting $w_0$, $w_1$ and $\epsilon$ (see text).}
  \label{tab:quantitative_results_estimation}
\begin{adjustbox}{max width=\textwidth}
\begin{tabular}{ccccccccccccc}
&&\multicolumn{10}{c}{\textbf{Image Setting}}& \\
\cmidrule(lr){3-12}
&&
\multicolumn{2}{c}{GT1} &
\multicolumn{2}{c}{GT2} &
\multicolumn{2}{c}{GT3} &
\multicolumn{2}{c}{GT4} &
\multicolumn{2}{c}{GT5} &\\
\cmidrule(lr){3-4} \cmidrule(lr){5-6} \cmidrule(lr){7-8} \cmidrule(lr){9-10}\cmidrule(lr){11-12}
 \multicolumn{2}{c}{\textbf{Method}} & IID & Brodatz  & IID & Brodatz  & IID & Brodatz  & IID & Brodatz & IID & Brodatz  & \textbf{Time (s)}\\\midrule[1pt]

\multirow{3}{*}{\shortstack[c]{\textit{Algebraic}}}
&\shortstack[c]{GT $w_0,w_1,\epsilon$} &
0.000 & 0.003 & 0.001 & 0.028 & 0.000 & 0.006 & 0.001 & 0.030 & 0.000 & 0.007&0.23
\\ 
    
&\shortstack[c]{$w_0, w_1 = 0.5, \epsilon = 0.5 \rho$} &
0.000 & 0.003 & 0.019 & 0.064 & 0.001 & 0.007 & 0.016 & 0.058 & 0.002 & 0.011 &0.14
\\

&\shortstack[c]{Search $w_0,w_1,\epsilon$} &
0.000 & 0.010 & 0.002 & 0.040 & 0.002 & 0.017 & 0.003 & 0.032 & 0.003  & 0.021 & 7.81
\\ 

\midrule[1pt]
\multirow{3}{*}{\shortstack[c]{\textit{Spectral}}}
&\shortstack[c]{GT $w_0,w_1,\epsilon$} &
0.000 & 0.002 & 0.001 & 0.034 & 0.000 & 0.006 & 0.001 & 0.032 & 0.000 & 0.007 &0.06
\\ 

&\shortstack[c]{$w_0, w_1 = 0.5, \epsilon = 0.5 \rho$} &
0.000 & 0.003 & 0.019 & 0.067 & 0.001 & 0.007 & 0.016 & 0.059 & 0.002 & 0.011 &0.06
\\

&\shortstack[c]{Search $w_0,w_1,\epsilon$} &
0.000 & 0.008 & 0.001 & 0.034 & 0.002 & 0.016 & 0.003 & 0.040 & 0.003  & 0.020 &0.16
    \\ \midrule[1pt]
\multirow{4}{*}{{ALT}}
& $\lambda = 1$  
 & 0.044  & 0.091  & 0.000  & 0.112  & 0.000  & 0.079  & 0.000  & 0.098 & 0.000  & 0.083 & 3.70 \\ 
& $\lambda = 3$   
 & 0.043  & 0.019  & 0.000  & 0.035  & 0.000  & 0.021  & 0.000  & 0.033 & 0.000  & 0.026 & 3.26 \\ 
& $\lambda = 5$   
 & 0.043  & 0.020  & 0.005  & 0.014  & 0.000  & 0.004  & 0.032  & 0.019 & 0.042  & 0.016 & 3.29 \\ 
& $\lambda = 10$  
 & 0.043  & 0.073  & 0.031  & 0.033  & 0.027  & 0.008  & 0.032  & 0.046 & 0.042  & 0.055 & 3.04
\end{tabular}
\end{adjustbox}
\end{table}

\subsection{Segmentation Evaluation on Synthetic Images}

After estimating appearance models using either the algebraic or
spectral methods we compute segmentations by minimizing Equation
(\ref{eq:E_seg}) using a max-flow/min-cut algorithm
(\cite{greig1989exact,boykov1999fast}).  We compared this approach to
several texture segmentation methods.

The methods we compare to include Level Set Segmentation using
Wasserstein Distances (LSWD) \cite{ni2009local}, Images as Occlusions
of Textures (ORTSEG) \cite{mccann2014images} and Factorization Based
Segmentation (FBS) \cite{yuan2015factorization}.  For each of these
methods, we used the Matlab implementations provided by the authors.
We tuned the parameters of each method to improve their performance in
our dataset.  We also evaluate the segmentation results obtained with
the iterative scheme ALT described above.

All of the methods we have used for comparison assume either
explicitly or implicitly that regions have homogeneous appearance.
FBS uses a filter bank to define local features, while LSWD and ORTSEG
work with raw pixel values.  LSWD, ORTSEG and FBS require the
selection of a window size parameter that has a function similar to
$r$ in our methods.

We used $\rho = 0.06$ to estimate appearance models with our methods.
We set $w_0$, $w_1$ and $\epsilon$ by searching over the parameters
explicitly (see Section~\ref{sec:w0w1e}).  We compute segmentations
using several choices for $\lambda$ in Equation (\ref{eq:E_seg}) and
evaluate each choice separately.

Figure~\ref{fig:quantitative_results_visuals} illustrates some of the
segmentations obtained using the different methods for both types of
images (IID and Brodatz) used for evaluation.  In these examples we
used $\lambda=5$ to compute segmentations with our methods and in ALT.

Table \ref{tab:quantitative_results_segmentation} provides a
quantitative evaluation on the full set of synthetic images generated
using the procedure described in Section~\ref{sec:synt}.  This is the
same data used to generate the results in Table
\ref{tab:quantitative_results_estimation}.  Notice that the runtime of
our methods is increased for the segmentation experiments
(Table~\ref{tab:quantitative_results_segmentation}) when compared to
the model estimation experiments
(Table~\ref{tab:quantitative_results_estimation}) due to the addition
of the graph cut computation to obtain a segmentation after estimating
appearance models.

Table~\ref{tab:quantitative_results_segmentation} demonstrates a
clear advantage of our methods under the IID case. For the Brodatz
setting, the results demonstrate that our methods provide
high quality segmentations without relying on
iterative approaches and filter banks.  This makes our methods
faster than most of the other approaches, while still leading to
accurate results.

These results confirm the soundness of the assumptions presented in
Section \ref{sec:model}.  They also confirm the efficacy of our
methods for segmenting images with complex textures, despite the fact
that we work directly with raw pixel values. As can be seen in
Figure~\ref{fig:quantitative_results_visuals} this leads to
segmentations that are accurate near region boundaries, where methods
that rely on filter responses often suffer.  Finally, although not
presented here for the sake of simplicity, our methods could have
their segmentation performance further improved when using the
estimated appearance as an initial guess for an iterative scheme such
as ALT.

\begin{figure}[t]
    \centering
    \begin{subfigure}[t]{0.13\textwidth}
        \includegraphics[width=\linewidth]{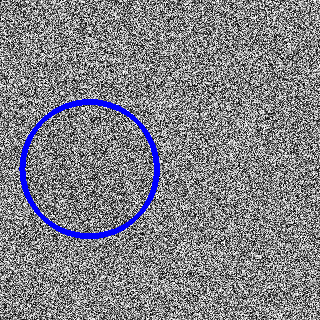}
    \end{subfigure}
    \begin{subfigure}[t]{0.13\textwidth}
        \includegraphics[width=\linewidth]{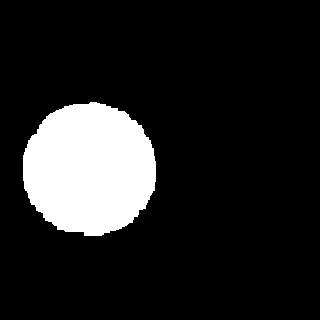}
    \end{subfigure}    
    \begin{subfigure}[t]{0.13\textwidth}
        \includegraphics[width=\linewidth]{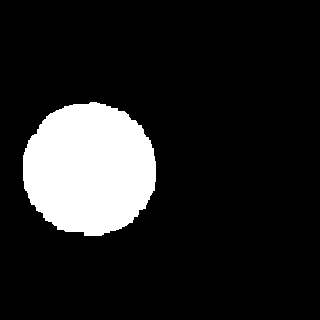}
    \end{subfigure}    
    \begin{subfigure}[t]{0.13\textwidth}
        \includegraphics[width=\linewidth]{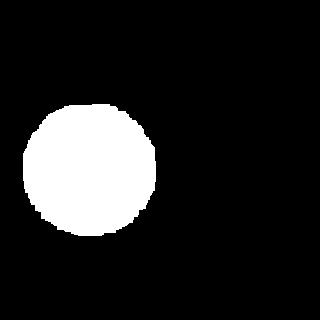}
    \end{subfigure}
    \begin{subfigure}[t]{0.13\textwidth}
        \includegraphics[width=\linewidth]{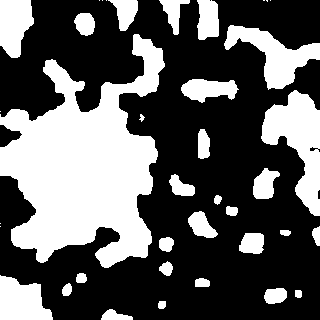}
    \end{subfigure}
    \begin{subfigure}[t]{0.13\textwidth}
        \includegraphics[width=\linewidth]{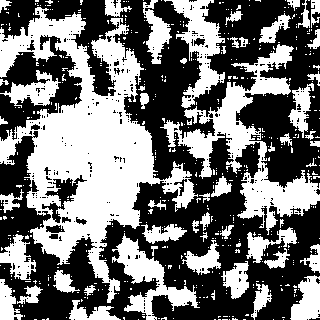}
    \end{subfigure}
    \begin{subfigure}[t]{0.13\textwidth}
        \includegraphics[width=\linewidth]{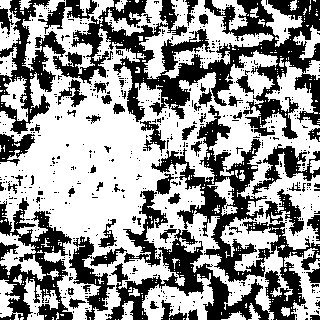}
    \end{subfigure}
    
    \begin{subfigure}[t]{0.13\textwidth}
        \includegraphics[width=\linewidth]{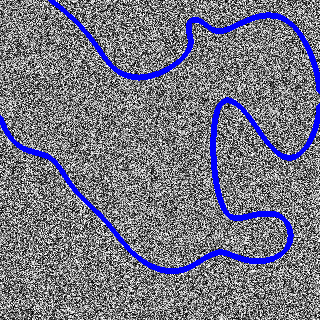}
    \end{subfigure}
    \begin{subfigure}[t]{0.13\textwidth}
        \includegraphics[width=\linewidth]{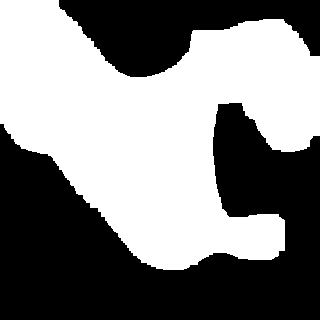}
    \end{subfigure}    
    \begin{subfigure}[t]{0.13\textwidth}
        \includegraphics[width=\linewidth]{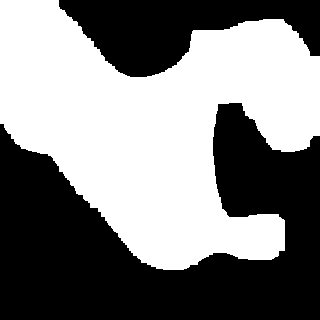}
    \end{subfigure}    
    \begin{subfigure}[t]{0.13\textwidth}
        \includegraphics[width=\linewidth]{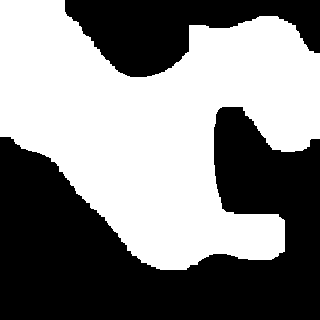}
    \end{subfigure}
    \begin{subfigure}[t]{0.13\textwidth}
        \includegraphics[width=\linewidth]{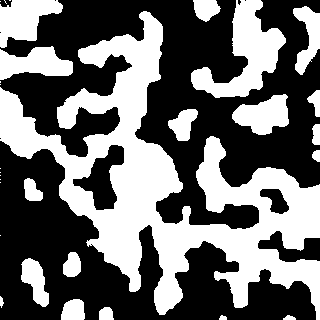}
    \end{subfigure}
    \begin{subfigure}[t]{0.13\textwidth}
        \includegraphics[width=\linewidth]{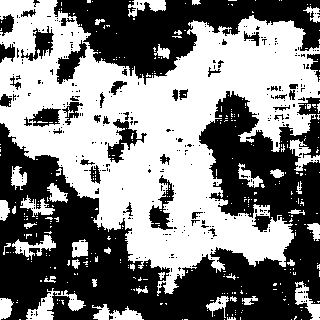}
    \end{subfigure}
    \begin{subfigure}[t]{0.13\textwidth}
        \includegraphics[width=\linewidth]{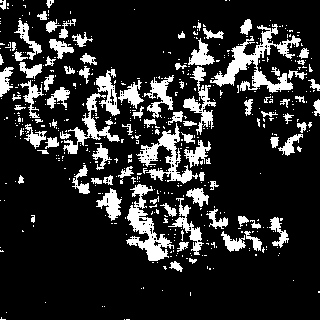}
    \end{subfigure}

    \begin{subfigure}[t]{0.13\textwidth}
        \includegraphics[width=\linewidth]{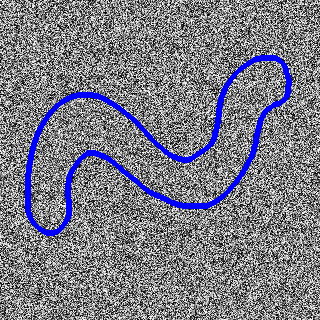}
    \end{subfigure}
    \begin{subfigure}[t]{0.13\textwidth}
        \includegraphics[width=\linewidth]{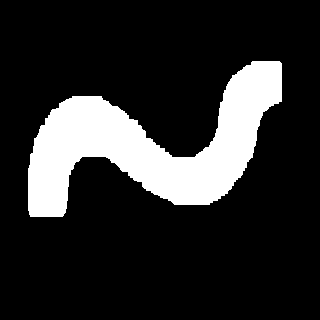}
    \end{subfigure}    
    \begin{subfigure}[t]{0.13\textwidth}
        \includegraphics[width=\linewidth]{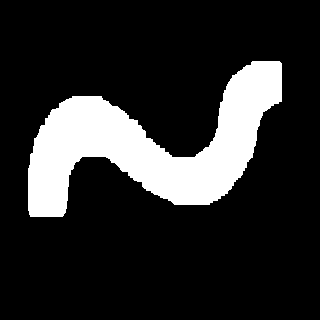}
    \end{subfigure}    
    \begin{subfigure}[t]{0.13\textwidth}
        \includegraphics[width=\linewidth]{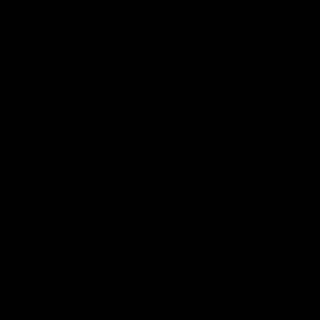}
    \end{subfigure}
    \begin{subfigure}[t]{0.13\textwidth}
        \includegraphics[width=\linewidth]{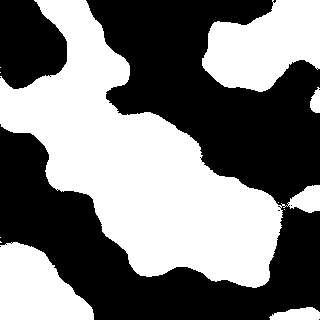}
    \end{subfigure}
    \begin{subfigure}[t]{0.13\textwidth}
        \includegraphics[width=\linewidth]{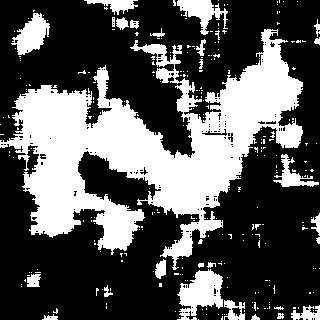}
    \end{subfigure}
    \begin{subfigure}[t]{0.13\textwidth}
        \includegraphics[width=\linewidth]{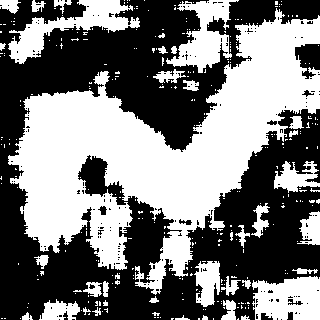}
    \end{subfigure}
    
    \begin{subfigure}[t]{0.13\textwidth}
        \includegraphics[width=\linewidth]{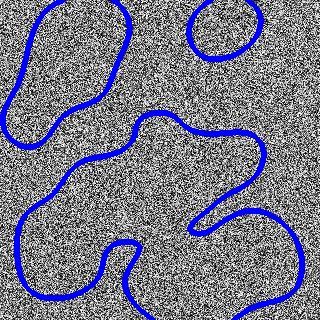}
    \end{subfigure}
    \begin{subfigure}[t]{0.13\textwidth}
        \includegraphics[width=\linewidth]{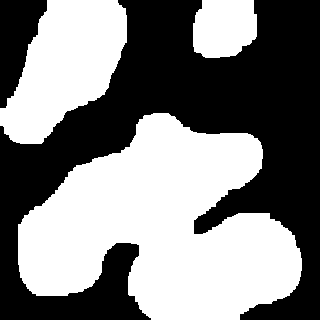}
    \end{subfigure}    
    \begin{subfigure}[t]{0.13\textwidth}
        \includegraphics[width=\linewidth]{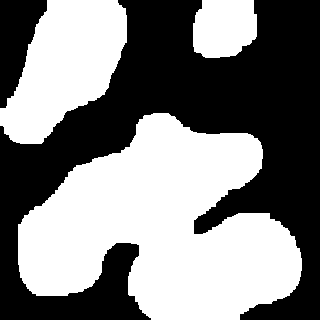}
    \end{subfigure}    
    \begin{subfigure}[t]{0.13\textwidth}
        \includegraphics[width=\linewidth]{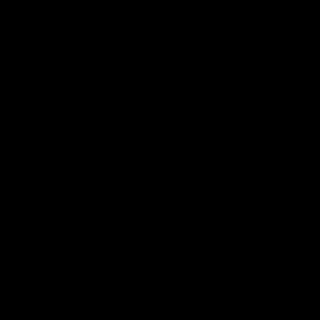}
    \end{subfigure}
    \begin{subfigure}[t]{0.13\textwidth}
        \includegraphics[width=\linewidth]{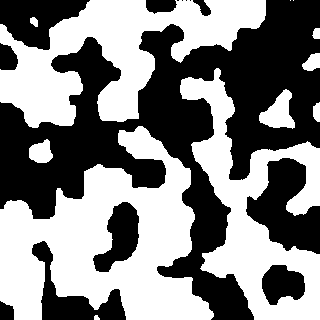}
    \end{subfigure}
    \begin{subfigure}[t]{0.13\textwidth}
        \includegraphics[width=\linewidth]{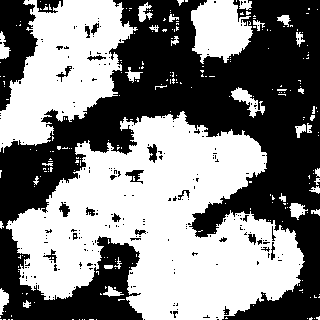}
    \end{subfigure}
    \begin{subfigure}[t]{0.13\textwidth}
        \includegraphics[width=\linewidth]{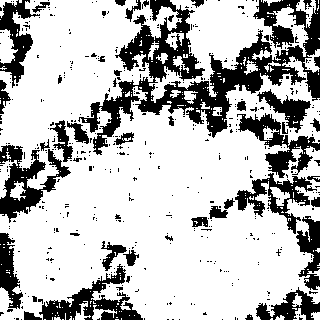}
    \end{subfigure}

    \begin{subfigure}[t]{0.13\textwidth}
        \includegraphics[width=\linewidth]{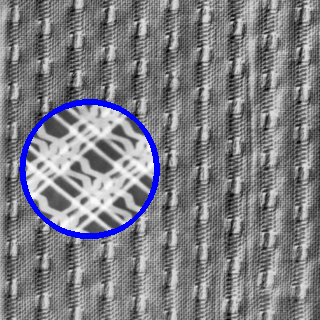}
    \end{subfigure}
    \begin{subfigure}[t]{0.13\textwidth}
        \includegraphics[width=\linewidth]{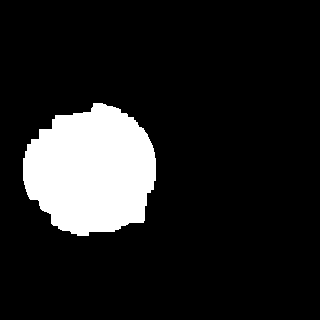}
    \end{subfigure}    
    \begin{subfigure}[t]{0.13\textwidth}
        \includegraphics[width=\linewidth]{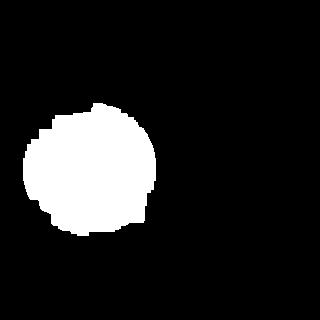}
    \end{subfigure}    
    \begin{subfigure}[t]{0.13\textwidth}
        \includegraphics[width=\linewidth]{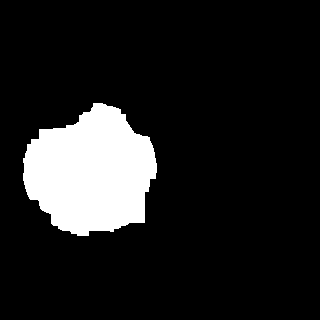}
    \end{subfigure}
    \begin{subfigure}[t]{0.13\textwidth}
        \includegraphics[width=\linewidth]{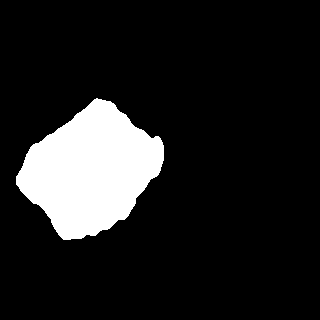}
    \end{subfigure}
    \begin{subfigure}[t]{0.13\textwidth}
        \includegraphics[width=\linewidth]{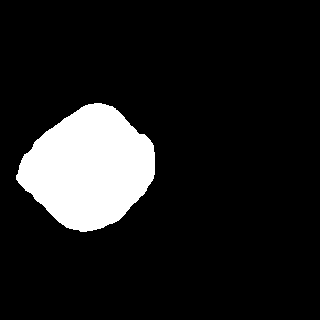}
    \end{subfigure}
    \begin{subfigure}[t]{0.13\textwidth}
        \includegraphics[width=\linewidth]{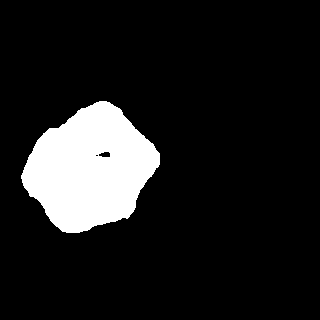}
    \end{subfigure}
    
    \begin{subfigure}[t]{0.13\textwidth}
        \includegraphics[width=\linewidth]{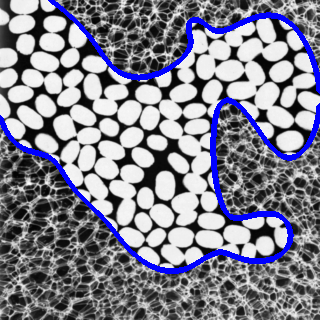}
    \end{subfigure}
    \begin{subfigure}[t]{0.13\textwidth}
        \includegraphics[width=\linewidth]{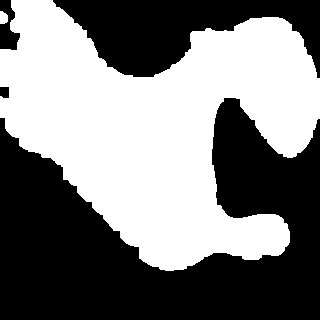}
    \end{subfigure}    
    \begin{subfigure}[t]{0.13\textwidth}
        \includegraphics[width=\linewidth]{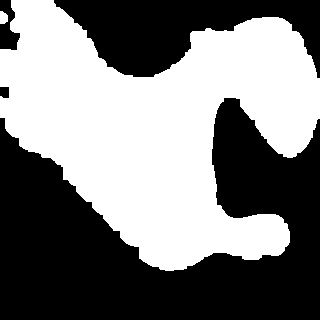}
    \end{subfigure}    
    \begin{subfigure}[t]{0.13\textwidth}
        \includegraphics[width=\linewidth]{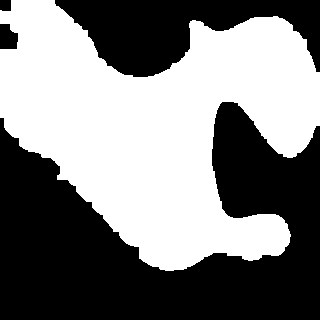}
    \end{subfigure}
    \begin{subfigure}[t]{0.13\textwidth}
        \includegraphics[width=\linewidth]{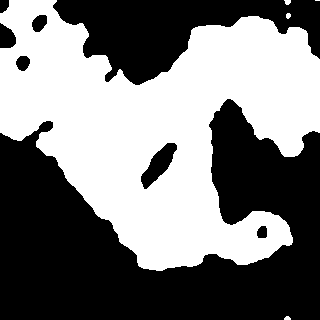}
    \end{subfigure}
    \begin{subfigure}[t]{0.13\textwidth}
        \includegraphics[width=\linewidth]{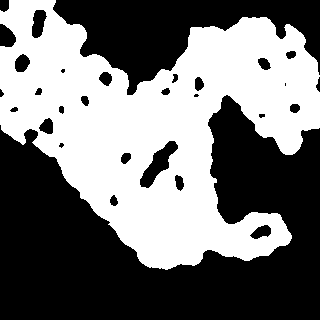}
    \end{subfigure}
    \begin{subfigure}[t]{0.13\textwidth}
        \includegraphics[width=\linewidth]{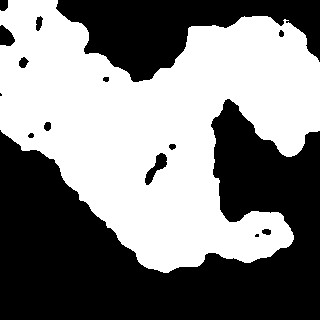}
    \end{subfigure}
    
    \begin{subfigure}[t]{0.13\textwidth}
            \includegraphics[width=\linewidth]{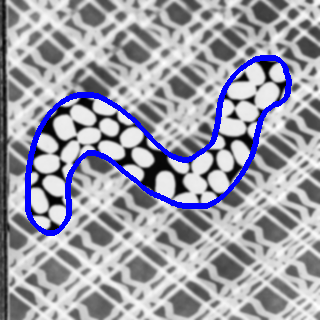}
    \end{subfigure}
    \begin{subfigure}[t]{0.13\textwidth}
            \includegraphics[width=\linewidth]{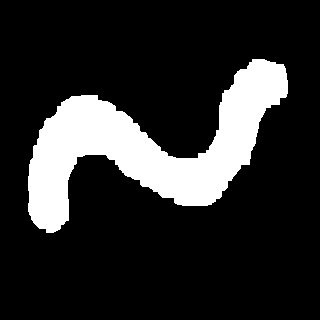}
    \end{subfigure}
    \begin{subfigure}[t]{0.13\textwidth}
            \includegraphics[width=\linewidth]{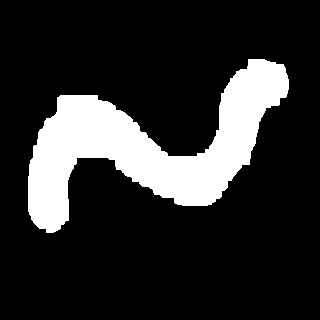}
    \end{subfigure}
    \begin{subfigure}[t]{0.13\textwidth}
        \includegraphics[width=\linewidth]{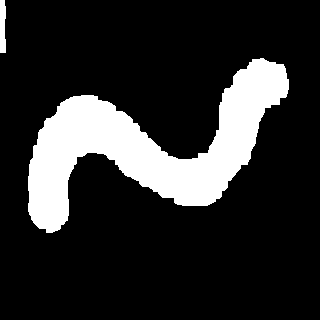}
    \end{subfigure}
    \begin{subfigure}[t]{0.13\textwidth}
        \includegraphics[width=\linewidth]{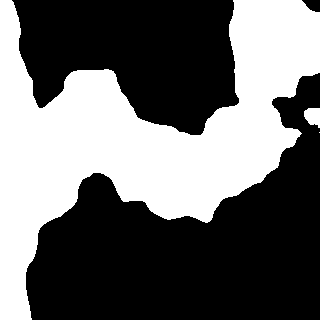}
    \end{subfigure}
    \begin{subfigure}[t]{0.13\textwidth}
        \includegraphics[width=\linewidth]{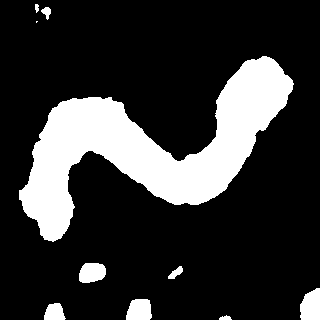}
    \end{subfigure}
    \begin{subfigure}[t]{0.13\textwidth}
        \includegraphics[width=\linewidth]{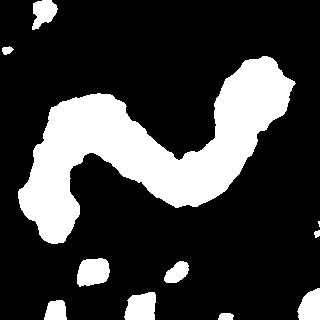}
    \end{subfigure}
    
    \begin{subfigure}[t]{0.13\textwidth}
            \includegraphics[width=\linewidth]{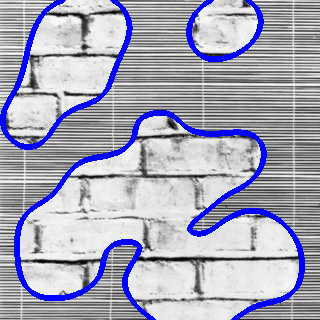}
            \caption{Original}
    \end{subfigure}
    \begin{subfigure}[t]{0.13\textwidth}
            \includegraphics[width=\linewidth]{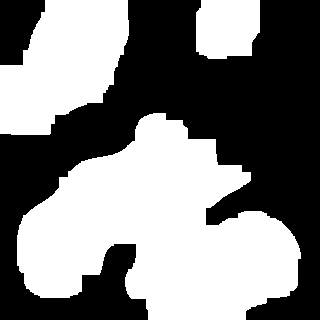}
            \caption{Algebraic}
    \end{subfigure}
    \begin{subfigure}[t]{0.13\textwidth}
            \includegraphics[width=\linewidth]{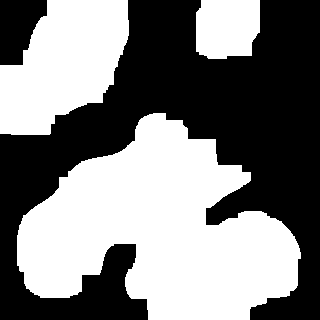}
            \caption{Spectral}
    \end{subfigure}
    \begin{subfigure}[t]{0.13\textwidth}
        \includegraphics[width=\linewidth]{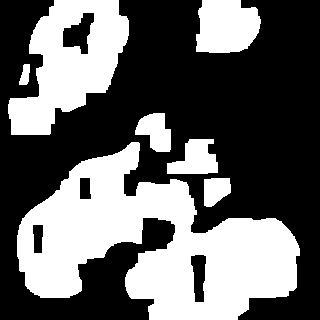}
        \caption{ALT}
    \end{subfigure}
    \begin{subfigure}[t]{0.13\textwidth}
        \includegraphics[width=\linewidth]{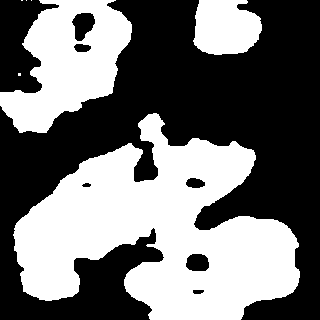}
        \caption{LSWD}
    \end{subfigure}
    \begin{subfigure}[t]{0.13\textwidth}
        \includegraphics[width=\linewidth]{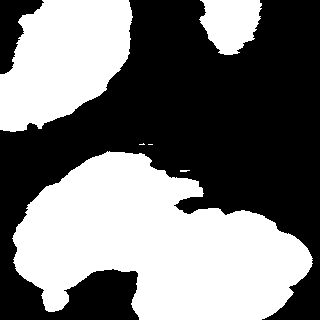}
        \caption{ORTSEG}
    \end{subfigure}
    \begin{subfigure}[t]{0.13\textwidth}
        \includegraphics[width=\linewidth]{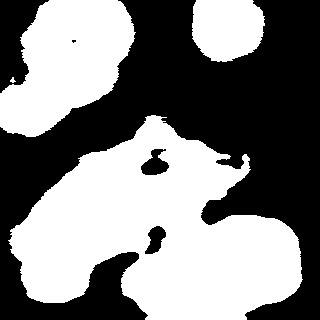}
        \caption{FBS}
    \end{subfigure}

    \caption{Selected segmentation results comparing our algorithms in
      (b) and (c) to other methods.}
    \label{fig:quantitative_results_visuals}
\end{figure}

\begin{table}
\setlength{\fboxrule}{.5pt}
\setlength{\fboxsep}{0pt}
\centering
\caption{Average $\jac$ index (higher is better) of different segmentation methods on the
  synthetic data generated using different segmentation masks.}
  \label{tab:quantitative_results_segmentation}
\begin{adjustbox}{max width=\textwidth}
\begin{tabular}{ccccccccccccc}
&&\multicolumn{10}{c}{\textbf{Image Setting}}& \\
\cmidrule(lr){3-12}
&&
\multicolumn{2}{c}{GT1} &
\multicolumn{2}{c}{GT2} &
\multicolumn{2}{c}{GT3} &
\multicolumn{2}{c}{GT4} &
\multicolumn{2}{c}{GT5} &\\
\cmidrule(lr){3-4} \cmidrule(lr){5-6} \cmidrule(lr){7-8} \cmidrule(lr){9-10}\cmidrule(lr){11-12}
\textbf{Method} & $\lambda$ & IID & Brodatz  & IID & Brodatz  & IID & Brodatz  & IID & Brodatz & IID & Brodatz  & \textbf{Time (s)}\\\midrule[1pt]
 
\multirow{4}{*}{\textit{Algebraic}}
 & 3 
 &1.000  & 0.896  & 0.991  & 0.789  & 0.977  & 0.869  & 0.980  & 0.780  & 0.956  & 0.850 &  8.13 \\ 
 & 5 
 & 1.000  & 0.919  & 0.989  & 0.814  & 0.955  & 0.873  & 0.966  & 0.787 & 0.884  & 0.841 & 8.15 \\ 
 & 7 
 & 1.000  & 0.937  & 0.986  & 0.838  & 0.919  & 0.856  & 0.886  & 0.782 & 0.708  & 0.811 &  8.22 \\ 
 & 10
 & 1.000  & 0.936  & 0.976  & 0.860  & 0.793  & 0.837  & 0.826  & 0.777 & 0.570  & 0.767 & 8.31 \\ 
 \midrule[1pt]
\multirow{4}{*}{\textit{Spectral}}
 & 3  
 & 1.000  & 0.894  & 0.991  & 0.777  & 0.977  & 0.863  & 0.980  & 0.780& 0.958  & 0.846 & 0.53 \\ 
 & 5  
 & 1.000  & 0.918  & 0.989  & 0.805  & 0.957  & 0.862  & 0.965  & 0.795& 0.898  & 0.841 & 0.57 \\ 
 & 7  
 & 1.000  & 0.934  & 0.986  & 0.819  & 0.918  & 0.845  & 0.884  & 0.783& 0.759  & 0.808 & 0.64 \\ 
 & 10  
 & 1.000  & 0.930  & 0.974  & 0.850  & 0.804  & 0.824  & 0.823  & 0.783& 0.590  & 0.771 & 0.72 \\ 
 \midrule[1pt]

\multirow{4}{*}{{ALT}}
& 1  
 & 0.467  & 0.683  & 0.990  & 0.591  & 0.986  & 0.700  & 0.984  & 0.620& 0.982  & 0.694  & 3.70 \\ 
& 3  
 & 0.500  & 0.874  & 0.991  & 0.761  & 0.981  & 0.854  & 0.977  & 0.757& 0.974  & 0.820 & 3.26 \\ 
& 5  
 & 0.500  & 0.874  & 0.858  & 0.843  & 0.966  & 0.904  & 0.187  & 0.785& 0.500  & 0.835 & 3.29 \\ 
& 10 
 & 0.500  & 0.673  & 0.141  & 0.784  & 0.551  & 0.842  & 0.177  & 0.631& 0.500  & 0.668  & 3.04 \\ \midrule[1pt]
LSWD  
& --  & 0.936  & 0.959  & 0.602  & 0.737  & 0.805  & 0.844  & 0.576  & 0.669  & 0.718  & 0.776 & 89.35 \\ \midrule[1pt]

ORTSEG  
& -- & 0.804  & 0.935  & 0.785  & 0.773  & 0.761  & 0.883  & 0.766  & 0.762& 0.719  & 0.851 & 1.53 \\ \midrule[1pt]

FBS 
& -- & 0.585  & 0.908  & 0.582  & 0.734  & 0.549  & 0.842  & 0.581  & 0.700& 0.547  & 0.810 & 0.08 \\ 

\end{tabular}
\end{adjustbox}
\end{table}

\subsection{Real Images}\label{sec:realimgs}

We also tested the proposed algorithms on real images from a variety
of datasets, including the Berkeley Segmentation Dataset
\cite{martin01}, the Plant Seedlings Dataset
\cite{giselsson2017public} and a Scanning Electron Microscope (SEM)
dataset \cite{semdataset2018}. The images were chosen such that
Assumption \ref{as:homogeneity} approximately holds.

Figure \ref{fig:real_images} shows some of the results obtained using
our methods for estimating appearance models followed by segmentation
using graph cuts.  For each image, we used $\rho = 0.03$ and $\lambda
= 5$.  These results illustrate how the proposed algorithms work well
on a variety of different types of images.

For these experiments we added a pre-processing step to our algorithms
to reduce the total number of colors in RGB images to a smaller number
of quantized values.  This is necessary in order to obtain good
estimates for $\alpha$ and $\beta$ on images with limited resolution.

To quantize the colors in an RGB image we repeatedly partition the
color space until each partition has at most 1000 pixels.  Starting
from the whole set of pixels, we partition the set into two using a
random hyperplane in RGB space going through the center of mass of the
set.  We recurse this procedure until the stopping criteria is met.
The same approach could be used for vector valued images such as
hyperspectral images that arise in remote sensing applications.

\begin{figure}[t]
    \centering
    \begin{subfigure}[t]{\linewidth}
        \centering
        \includegraphics[width=.239\linewidth]{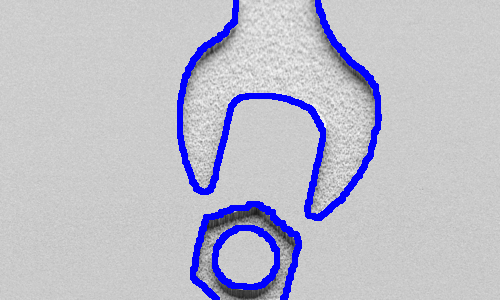}
        \includegraphics[width=.239\linewidth]{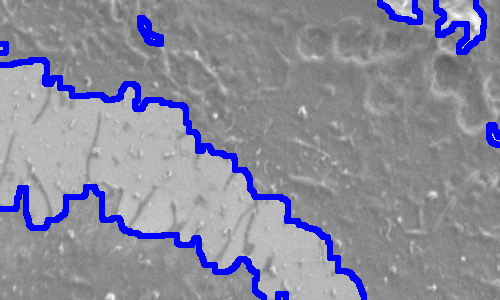}
        \includegraphics[width=.239\linewidth]{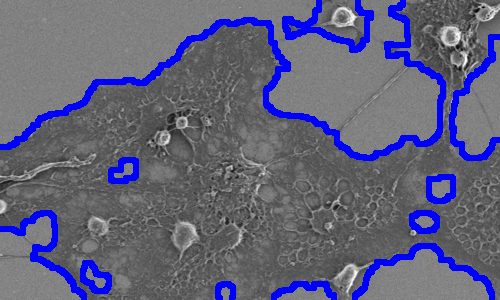}
        \includegraphics[width=.239\linewidth]{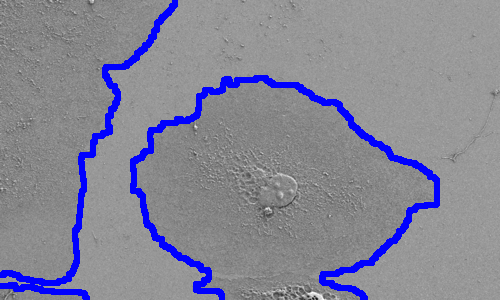}\\[5pt]
        \includegraphics[width=.239\linewidth]{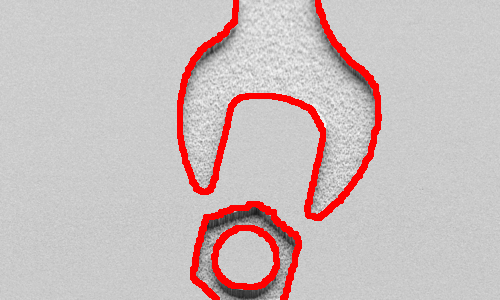}
        \includegraphics[width=.239\linewidth]{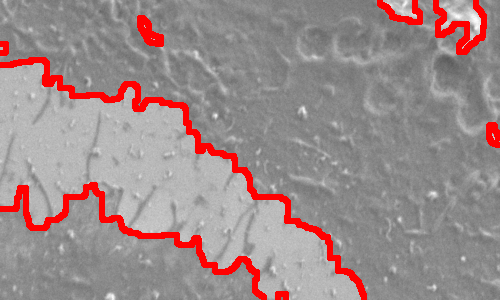}
        \includegraphics[width=.239\linewidth]{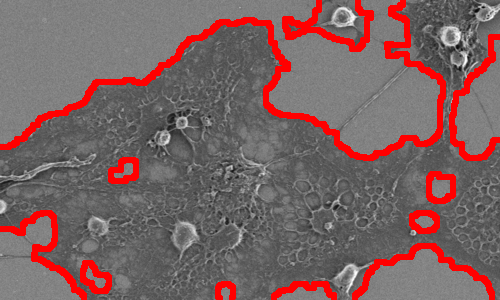}
        \includegraphics[width=.239\linewidth]{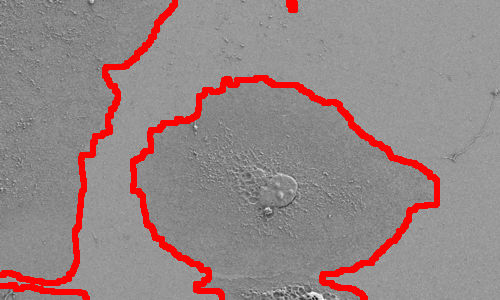}
        \caption{\cite{semdataset2018}}
    \end{subfigure}
    
    \vspace{.3cm}
    
    \begin{subfigure}[t]{\linewidth}
        \centering
        \includegraphics[height=.19\linewidth]{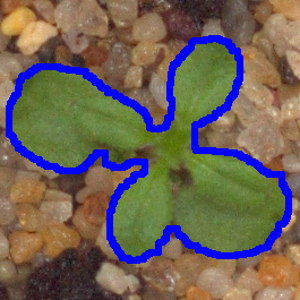}
        \includegraphics[height=.19\linewidth]{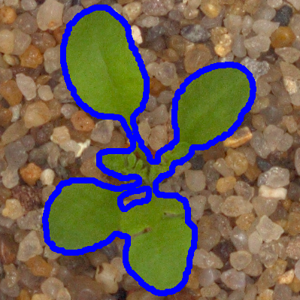}
        \includegraphics[height=.19\linewidth]{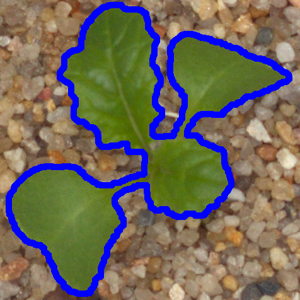}
        \includegraphics[height=.19\linewidth]{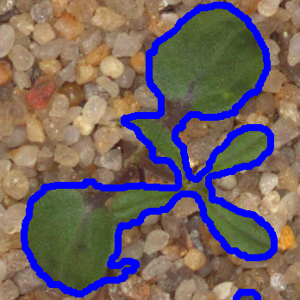}
        \includegraphics[height=.19\linewidth]{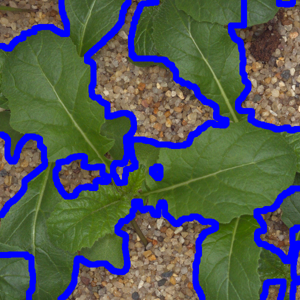}\\[5pt]
        \includegraphics[height=.19\linewidth]{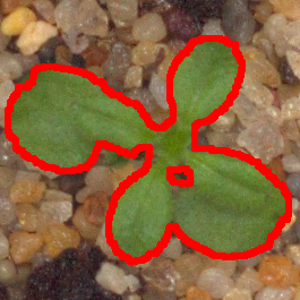}
        \includegraphics[height=.19\linewidth]{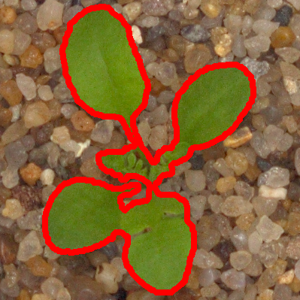}
        \includegraphics[height=.19\linewidth]{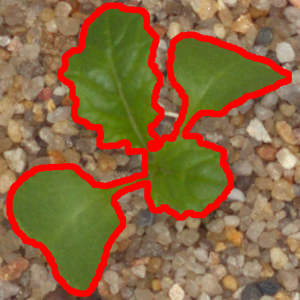}
        \includegraphics[height=.19\linewidth]{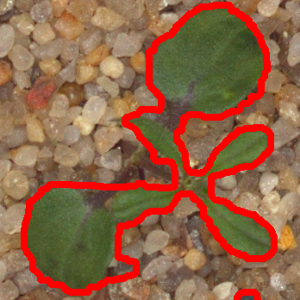}
        \includegraphics[height=.19\linewidth]{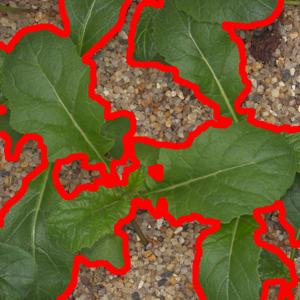}
        \caption{\cite{giselsson2017public}}
    \end{subfigure}
    
    \vspace{.3cm}
    
    \begin{subfigure}[t]{\linewidth}
        \centering
        \includegraphics[height=.185\linewidth]{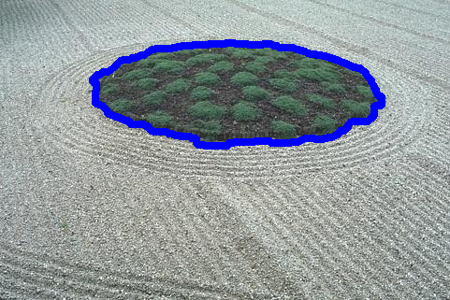}
        \includegraphics[height=.185\linewidth]{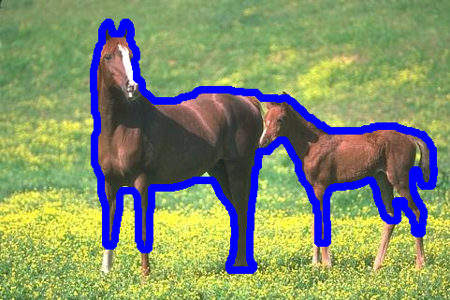}
        \includegraphics[height=.185\linewidth]{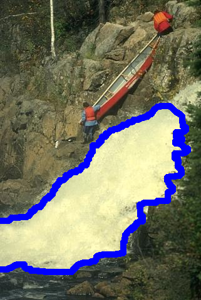}
        \includegraphics[height=.185\linewidth]{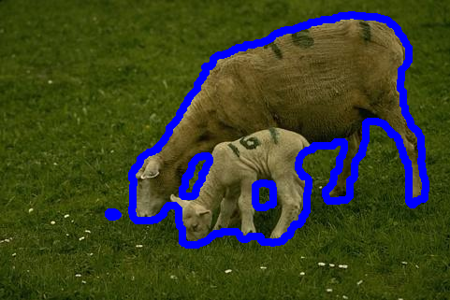}\\[5pt]
        \includegraphics[height=.185\linewidth]{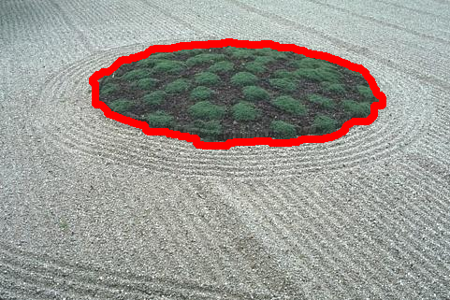}
        \includegraphics[height=.185\linewidth]{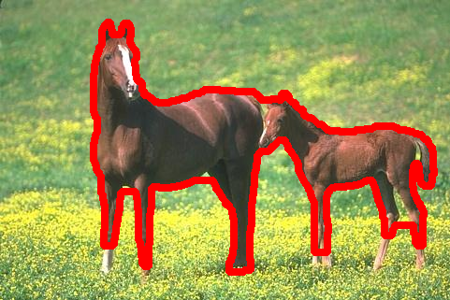}
        \includegraphics[height=.185\linewidth]{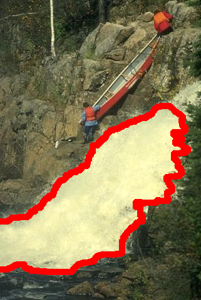}
        \includegraphics[height=.185\linewidth]{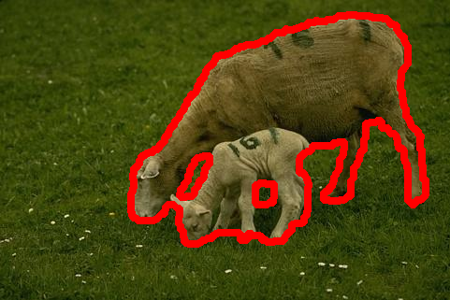}
        \caption{\cite{martin01}}
    \end{subfigure}
    \caption{Application our methods on real images from a variety of
      different datasets. The blue and red contours are the results of
      segmentation using appearance models estimated using the
      algebraic and spectral methods,
      respectively.} \label{fig:real_images}
\end{figure}

\clearpage

\section{Conclusion}

Many image segmentation algorithms rely on appearance models to
classify pixels into different regions.  We have shown that appearance
models can be estimated directly from an unsegmented image.  Our
approach is based on novel algebraic expressions that relate local
image statistics to the appearance models of different image
regions.  Our experiments demonstrate the algorithms we introduce in
this paper work well in a variety of settings.  The resulting
appearance models can be used to segment images of different types,
including textured images and other images where regions have complex
appearances.  These results also suggest that segmentation algorithms
can be improved by making use of second order pixel statistics.

Thus far we have focused on the problem of binary segmentation.
An interesting direction for future work would be to extend our
approach to images with more than two regions.  As discussed in
Section~\ref{sec:multi} it is possible to derive algebraic constraints
similar to the ones we have used to the case of images with (at most)
$m$ regions.  However, solving for appearance models using these
constraints remains a challenging computational problem.

\bibliographystyle{plain}
\bibliography{references}

\end{document}